\documentclass[letterpaper]{article} 
\usepackage{aaai23}  
\usepackage{times}  
\usepackage{helvet}  
\usepackage{courier}  
\usepackage[hyphens]{url}  
\usepackage{graphicx} 
\usepackage{natbib}  
\usepackage{caption} 
\title{AAAI Press Anonymous Submission\\Instructions for Authors Using \LaTeX{}}

\usepackage[utf8]{inputenc} 
\usepackage[T1]{fontenc}    
\usepackage{url}            
\usepackage{booktabs}       
\usepackage{amsfonts}       
\usepackage{nicefrac}       
\usepackage{microtype}      
\usepackage{comment}
\usepackage{fixltx2e}
\usepackage{lastpage}
\usepackage{colortbl}
\usepackage{graphicx,color}
\usepackage{clrscode}
\usepackage{subfigure}
\usepackage{amsmath}
\usepackage{amssymb}
\usepackage{amsthm}
\usepackage[dvipsnames]{xcolor}
\usepackage{algorithm}
\usepackage{url}
\usepackage{caption}
\usepackage[compact]{titlesec}

\usepackage[noend]{algpseudocode}

\DeclareUnicodeCharacter{2212}{-}
\usepackage{booktabs} 
\usepackage{adjustbox}
\usepackage{multirow}
\usepackage{longtable}
\usepackage{clrscode}
\usepackage{array}
\usepackage{wrapfig}
\usepackage{enumitem}
\usepackage{tabularx}
\usepackage{colortbl}

\setlist{nosep} 
\usepackage{pgfgantt}

\setganttlinklabel{s-s}{START-TO-START}
\setganttlinklabel{f-s}{FINISH-TO-START}
\setganttlinklabel{f-f}{FINISH-TO-FINISH}

\newcommand{\qq}[1]{\textcolor{NavyBlue}{#1}}
\newtheorem{fact}{Fact}[section]

\newtheorem{theorem}{Theorem}[section]

\newtheorem{lemma}[theorem]{Lemma}
\newtheorem{proposition}[theorem]{Proposition}

\newtheorem{definition}[theorem]{Definition}

\providecommand{\mypara}[1]{\smallskip\noindent\emph{#1} }
\providecommand{\myparab}[1]{\smallskip\noindent\textbf{#1} }

\newcommand{\Res}{\mathrm{Res}}

\newcommand{\Rank}{\mathrm{Rank}}

\newcommand{\mf}{\mathbf{f}}

\newcommand{\ms}{\mathbf{s}}
\newcommand{\mx}{\mathbf{x}}
\newcommand{\my}{\mathbf{y}}
\newcommand{\mz}{\mathbf{z}}
\newcommand{\mK}{\mathbf{K}}
\newcommand{\mS}{\mathbf{S}}
\newcommand{\mX}{\mathbf{X}}
\newcommand{\mY}{\mathbf{Y}}

\newcommand{\reals}{\mathbf{R}}

\newcommand{\calB}{\mathcal{B}}

\newcommand{\calH}{\mathcal{H}}

\newcommand{\calD}{\mathcal{D}}

\newcommand{\calK}{\mathcal{K}}
\newcommand{\calL}{\mathcal{L}}
\newcommand{\calP}{\mathcal{P}}

\newcommand{\calg}{\mathcal{G}}

\newcommand{\Esymb}{\mathbb{E}}
\DeclareMathOperator*{\E}{\Esymb}

\newcommand{\cale}{\mathcal{E}}

\newcommand{\calp}{\mathcal P}

\newcommand{\dist}{\mathrm{Dist}}

\usepackage{tikz}
\newcommand*\emptycirc[1][1ex]{\tikz\draw (0,0) circle (#1);} 

\newcommand*\fullcirc[1][1ex]{\tikz\fill (0,0) circle (#1);}

\newcommand{\transpose}{\mathrm{T}}

\newcommand{\Linpvel}{\textsc{Lin-PVEL}}

\usepackage{booktabs}

\date{}

\newcounter{noteMCctr} 

\newcommand{\mc}[1]{\textcolor{black}{#1}}

\newcounter{noteZZctr} 

\newcounter{noteZLctr} 

\newcounter{noteQWctr} 

\newcommand\sbullet[1][.5]{\mathbin{\vcenter{\hbox{\scalebox{#1}{$\bullet$}}}}}
\newcommand{\bb}{\hspace{-1mm} $\bullet$ }

\usepackage{etoolbox}
\makeatletter
\patchcmd\l@section{%
  \nobreak\hfil\nobreak
}{%
  \nobreak
  \leaders\hbox{%
    $\m@th \mkern \@dotsep mu\hbox{.}\mkern \@dotsep mu$%
  }%
  \hfill
  \nobreak
}{}{\errmessage{\noexpand\l@section could not be patched}}
\makeatother

\makeatletter
\renewcommand{\@seccntformat}[1]{}
\makeatother
\begin{document}
\title{Symphony in the Latent Space: Provably Integrating High-dimensional Techniques with Non-linear Machine Learning Models}


\author{Qiong Wu\thanks{Currently working at AT\&T Labs.}\textsuperscript{1}, 
Jian Li\textsuperscript{2}, 
Zhenming Liu\textsuperscript{1}, 
Yanhua Li\textsuperscript{3}, 
Mihai Cucuringu\textsuperscript{4}\\
\textsuperscript{1}{William \& Mary}\\
\textsuperscript{2}{Tsinghua University}\\
\textsuperscript{2}{Worcester Polytechnic Institute}\\
\textsuperscript{4}{University of Oxford and The Alan Turing Institute}\\}

\maketitle
\setcounter{page}{1}


\begin{abstract}

This paper revisits building machine learning algorithms that involve interactions between entities, such as those between financial assets in an actively managed portfolio, or interactions between users in a social network. Our goal is to forecast the future evolution of ensembles of multivariate time series in such applications (e.g., the future return of a financial asset or the future popularity of a Twitter account). Designing ML algorithms for such systems requires addressing the challenges of high-dimensional interactions and non-linearity. Existing approaches usually adopt an \textbf{ad-hoc} approach to integrating high-dimensional techniques into non-linear models and recent studies have shown these approaches have questionable efficacy in time-evolving interacting systems.

To this end, we propose a novel framework, which we dub as the  \emph{\textbf{additive influence model}}. Under our modeling assumption, we show that it is possible to decouple the learning of high-dimensional interactions from the learning of non-linear feature interactions. To learn the high-dimensional interactions, we leverage kernel-based techniques, with provable guarantees, to embed the entities in a low-dimensional latent space. To learn the non-linear feature-response interactions, we generalize prominent machine learning techniques, including designing a new statistically sound non-parametric method and an ensemble learning algorithm optimized for vector regressions. 
Extensive experiments on two common applications demonstrate that our new algorithms deliver significantly stronger forecasting power compared to standard and recently proposed methods. 




\end{abstract}

\vspace{-2mm}

\section{Introduction}

We revisit the problem of building machine learning algorithms that involve interactions between entities, such as those between users and items in a recommendation system, or between financial assets in an actively managed portfolio, or between populations in different counties in a disease-spreading process. 
Our proposed forecasting model uses information available up to time $t$ to predict $\my_{t+1,i}$, the future behavior of entity $i$ at time $t+1$  (e.g., the future price of stock $i$ at time $t+1$), for a total number of $d$ entities~\cite{laptev2017time,farhangi2022aa}. Designing such models has proven remarkably difficult, as one needs to circumvent two main challenges that require often incompatible solutions.  






\vspace{-1mm}
\myparab{1. Cross-entity interaction: high-dimensionality.} 
In many ensembles of multivariate time series systems, it is often the case that the current state of one entity could potentially impact the future state of another. When considering the equity market as an example,
Amazon's disclosure of its revenue change in cloud services could indicate that the revenues of other cloud providers (e.g.,
competitors) could also change.

The interaction is high-dimensional because the total possible number of interactions is usually much larger than the number of available observations. For example, in a portfolio of 3,000 stocks, the total number of potential links between pairs of stocks is $3,000 \times 3,000 \approx 10^7$, but we often have only 2,500 data points (e.g., 10 years of daily data), and thus capturing the cross-entity interactions becomes a very challenging problem.

\vspace{-1mm}
\myparab{2. Feature-response interactions: non-linearity.} Linear models are usually insufficient to characterize the relationship between the response/label and the available information (features), thus techniques beyond simple linear regressions are heavily needed.
\mc{For example, in a financial context, economic productivity is non-linear in temperature for most countries; similarly, electricity consumption is a nonlinear function of temperature, and modeling this relationship is crucial for pricing electricity derivative contracts.}
As shown in Fig.~\ref{fig:all_patterns}(a),
the \mc{existing} relevant learning models can be categorized into the following two groups. 


\mypara{1. Provable cross-entity models (CEM) for high-dimensionality.} 
Cross-entity models solve a vector regression problem $\my_{t+1} = f(\mx_t) + \xi_t$, to forecast the future behavior of all entities, where $\my_{t+1} \triangleq (y_{t+1, 1}, \dots , y_{t+1, d})$, and $\mx_t$ denotes the features of all entities, constructed from their historical data. Since the features of one entity can be used to predict the future behavior of another, CEMs have stronger expressive and predictive power. CEMs are both computationally and statistically challenging because we need to solve the ``high-dimensional'' (overparametrized) problem and mathematically understand the root cause of the overfitting. Extensive research has been undertaken to design regularization techniques~\cite{chen2013reduced,friedman2001elements,wu2021bats} to address the issue, and most algorithms in this category are linear and have theoretical guarantees.

\vspace{0mm}
\begin{figure*}[ht!]
    \centering
    \includegraphics[height=45mm, width=165mm]{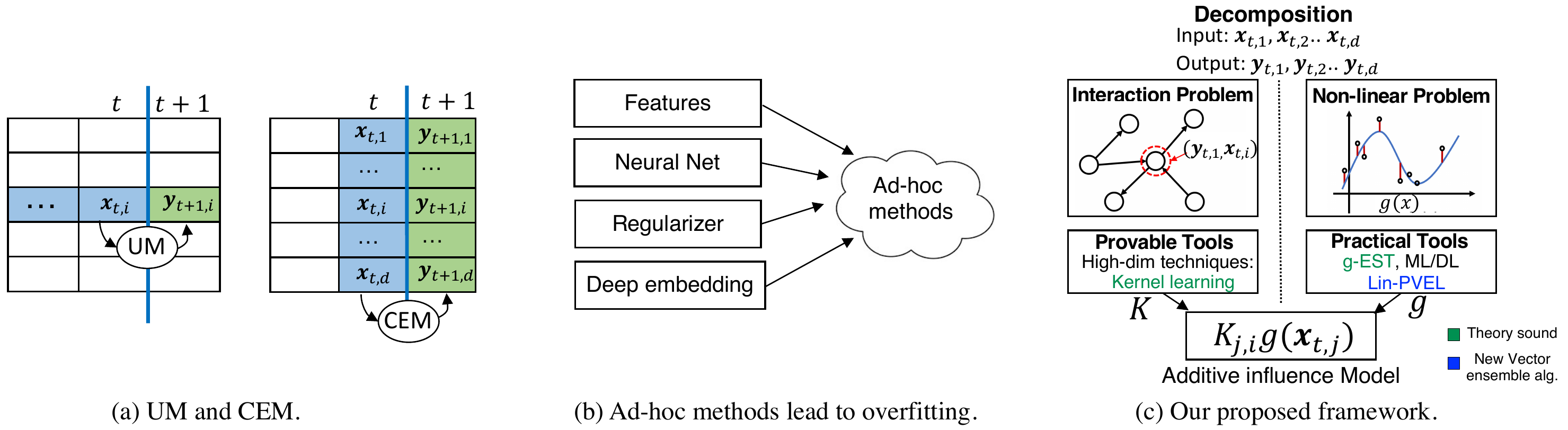}
    \vspace{-3mm}
    \captionsetup{font=small}
    \caption{(a) UM for non-linearity and CEM for high-dimensionality. (b) Exsiting ad-hoc methods have questionable efficay. (c) Our framework decouples the high-dimensional learning of entity interactions and non-linear learning of feature interactions.}
    \label{fig:all_patterns}
    \vspace{-7mm}
\end{figure*}


\mypara{2. Practical univariate models (UM) for non-linearity.}  
Univariate models fit a function $\my_{t+1, i} = f(\mx_{t, i}) + \xi_{t, i}$  to forecast one entity's feature behavior by using features constructed from that entity's historical data. Univariate models primarily learn the feature-response interaction by using off-the-shelf ML techniques such as Deep learning (DL)~\cite{abadi2016tensorflow,hochreiter1997long,wu2019speaking} or gradient boosted algorithms~\cite{chen2016xgboost,ke2017lightgbm,dorogush2018catboost}.
These practical models are effective in extracting non-linear signals
but they often do not come with theoretical guarantees.

\vspace{-1mm}
\myparab{Existing integration techniques: ad-hoc methods} 
It remains unclear how to integrate two seemingly incompatible modeling processes (i.e., UM and CEM)  with different design philosophies. 
In Fig.~\ref{fig:all_patterns} (b), we show that existing integration solutions  predominately \mc{follow an \textbf{ad-hoc} approach}, in part due to the belief that deep learning is the ``holly-grail'' for practical problems~\cite{sejnowski2018deep}.  For example, one often adds an $\ell_1$- or $\ell_2$-regularizer to a neural net's cost function, hoping such regularizers will also magically work in neural nets~\cite{abadi2016tensorflow,paszke2017automatic}. However, the mathematical properties of a provable technique often break when \mc{combined} into a neural net. 
\mc{Furthermore, latent embedding models have also been recently} introduced ~\cite{wang2019alphastock,feng2019temporal,chen2019investment}. The central idea is to project the entities into points in a low-dimensional space so that similar \mc{entities (i.e., stocks in the above works) are closer to each other in this} embedding. Because point interactions are more restrictive in the latent space, they have the potential to address the overfitting issues~\cite{wang2019alphastock}. 
However, these lines of work do not offer any theoretical guarantees and are often not robust in practice.  Recent studies have \mc{demonstrated that the efficacy of such ad-hoc approaches is questionable}
in many interacting systems~\cite{RecSys19Evaluation, Steffen@19DBLP,qiong20embedding}. 



\vspace{-1mm}
\myparab{Our approach \& contributions}
We propose a general latent position model dubbed as the  \emph{additive influence model} to enable us to seamlessly orchestrate mathematically rigorous high-dimensional  techniques with practically effective machine learning algorithms.  In Fig.~\ref{fig:all_patterns}(c), we show that it is possible to decouple the learning of high-dim interactions between entities from the learning of the non-linear signals.

We assume each entity is associated 
with an embedded position $\mz_i$ and at timestamp  $t$, entity $i$ is also associated with an unobserved signal $\ms_{i,t} \in \reals$ 
that is a function of $\mx_{i,t}$. We assume the generative model  
$\my_{i,t} = \sum_{i \leq j}\kappa(\mz_i, \mz_j) \ms_{j,t} + \epsilon_{i,t}$, where $\kappa(\mz_i, \mz_j)$ is a function that measures the interaction strength between $\mz_i$ and $\mz_j$, and can be any kernel function, such as a Gaussian kernel or simply an inner product, and $\epsilon_{i,t}$ denotes noise. \mc{Each entity could potentially influence $\my_{i,t}$}. The influence 
of $j$ on $i$ depends on the ``distance'' or ``similarity'' between $\mz_i$ and $\mz_j$. On the other hand, we assume $\ms_{j,t} = g(\mx_{j,t})$ for some $g(\cdot)$, so that the model captures high-dimensional interactions via $\mz_{i}$ and non-linearity via $g(\cdot)$. 




Our proposed model allows for feature interactions through $g(\cdot)$, and addresses the overfitting problem arising from entity interactions because  the distances (interaction strength) between entities are constrained by the latent Euclidean space: when both $(\mz_i - \mz_j)$ and $(\mz_j - \mz_k)$ are small, then $(\mz_i - \mz_k)$ is also small, and thus the degree of freedom for entity interactions becomes substantially smaller than $O(d^2)$. 

Our goal is to learn both the  $\mz_i$'s and $g(\cdot)$. We note that these two learning tasks can be \emph{decoupled}:  high-dimensional methods can be developed to \emph{provably} estimate the $\mz_i$'s \emph{without the knowledge of} $g(\cdot)$, and when estimates of $\mz_i$'s are given, an experiment-driven process can be used to learn $g(\cdot)$ by examining prominent machine learning methods such as neural nets and boosting. In other words, when we learn entity interactions, we do not need to be troubled by the overfitting problem escalated by fine-tuning $g(\cdot)$, and when we learn feature interactions, the generalization error will not be jeopardized by the curse of dimensionality from entity interactions.  

\hspace{3mm} \bb  To learn the $\mz_i$'s, we design a simple algorithm that uses low-rank approximation of $\my_t$'s covariance matrix to infer the closeness of the entities and develop a novel theoretical analysis based on recent techniques from high dimensionality and kernel learning~\cite{belkin2018approximation,tang2013universally,wu2019adaptive}.

\hspace{3mm}  \bb  To learn $g(\cdot)$, we generalize major machine learning techniques, including neural nets, non-parametric, and boosting methods, to the additive influence model when estimates of $\mz_i$'s are known. We specifically
develop a moment-based algorithm for non-parametric learning of $g(\cdot)$, and a computationally efficient boosting algorithm. 

\hspace{3mm} \bb  Finally, we perform extensive experiments on a major equity market and social network datasets to confirm the efficacy of our modeling approaches and analysis.

\section{Related work and comparison}
\vspace{-1mm}


Univariate machine learning models handle feature-response interactions and mostly rely on deep learning and GBRT \cite{goodfellow2016deep,wu2020comprehensive,goodfellow2016deep,wuthrich1998daily,chen2016xgboost,ke2017lightgbm,dorogush2018catboost,gong2017acoustic,yang2020computational,ding2015deep,zhang2017stock,feng2018deep,han2018firm,wu2020comprehensive,chen2019deep,kelly2019characteristics,ke2019deepgbm,chen2019investment,li2019individualized,wu2015early}. These models aim to optimize their empirical performance and limit theoretical investigations. 
Recent cross-entity models consider the high-dimensional interactions, where overfitting easily happens and theoretical justifications are essential to avoid spurious result in practice. Cross-entity models are mostly linear models ~\cite{bunea2011optimal,koltchinskii2011nuclear,negahban2011estimation,huang2019shrinking} that have theoretical guarantees, but they cannot effective for non-linear feature-response interactions. Efforts for building CEMs include~\cite{tibshirani1996regression,candes2008introduction,tao2009compressed,hoerl1970ridge,tsigler2020benign,liu2019near}.

\vspace{-1mm}
\mypara{Ad-hoc approach for integration.} 
Recent integrating solutions for high-dimensionality and nonlinearity challenges has been a frustrating endeavor, which we can call the ad-hoc approaches and many were shown to have questionable efficacy in interacting systems.
\emph{1. Deep learning + Lasso/Ridge}
For example, one~\cite{abadi2016tensorflow,paszke2017automatic} often adds an $l_1$- or $l_2$-regularizer to a neural net’s cost function, hoping these regularizers can also magically work in neural nets. 
\emph{2. Deep embedding.} Recent studies have addressed high-dimensional entity interactions by using deep embedding, based on the idea that when entities are embedded in low-dim Euclidean space, they can interact in a quite restricted way, therefore preventing overfitting~\cite{zhao2020characterizing,shen2022learning,xie2016unsupervised,zhang2017stock,hu2018listening,li2019individualized,wang2019alphastock}. While this idea is effective for linear models~\cite{abraham2015low,li2017world}, deep embedding-based solutions may have very high false positive rates, for instance, when forecasting the returns of financial assets~\cite{qiong20embedding,wang2019alphastock}.

\myparab{Remark:}
\emph{(i) Modeling framework.} Our framework proposes a key algorithmic insight that the latent position estimation should be decoupled from the learning link function $g(\cdot)$.  We develop the first algorithm that can provably estimate the entity’s latent positions and provide theoretical guarantees.
Our novel analysis leverages a diverse set of tools from kernel learning, non-parametric methods, and random walks.
\emph{(ii) Comparison to deep embedding.} While embedding can be learned by deep learning~\cite{hu2018listening,wang2019alphastock}, it usually does not provide any theoretical guarantee, whereas our framework makes stricter assumptions (e.g., how embedding and features should interact) and delivers a quality guarantee. Deep embedding also requires every component including the function $g(\cdot)$ 
in the architecture to be represented by a neural net to run SGD, whereas we allow $g(\cdot)$ to be learned by a wide range of algorithms such as boosting or non-parametric techniques.

\section{Problem definition}
\label{sec:prelimary}
\vspace{-1mm}
\myparab{Notations.} 
For a matrix $A$, $\calP_r(A)$ denotes its rank-$r$ approximation obtained by keeping the top $r$ singular values and the corresponding singular vectors.  $\sigma_i(A)$ (resp. $\lambda_i(A)$) is the $i$-th singular value (resp. eigenvalue) of $A$. 
We use Python/MATLAB notation when we refer to a specific row or column. For example, $A_{1, :}$ is the first row of $A$, and $A_{:, 1}$ is the first column. $\|A\|_F$ and $\|A\|_2$  denote the Frobenius and spectral norms, respectively, of $A$.  In general, we use boldface upper case (e.g., $\mX$) to denote data matrices and boldface lower case (e.g., $\mx$) to denote one sample.  $\mx_{t, i}$, which refers to the features associated with stock $i$ at time $t$, can be one or multi-dimensional. Let  $(\mx_{t, i})_j$ be the $j$-th coordinate (feature) of $\mx_{t, i}$. An event occurring with high probability (whp) means that it happens with probability $ \geq 1-n^{-10}$, where 10 is an arbitrarily chosen large constant and is not optimized. A bivariate function is a Gaussian kernel if $\kappa(\mx, \mx') = \exp(-\|\mx - \mx'\|^2/\sigma^2)$, an inverse multi-quadratic (IMQ) kernel if $\kappa(\mx, \mx') = (c^2 + \|\mx - \mx'\|^2)^{-\alpha}$  ($\alpha > 0$), and an inner product kernel if $\kappa(\mx, \mx') = \langle \mx, \mx'\rangle$. 
A function $g(\cdot)$ is Lipschitz-continuous if $|g(\mx_1) - g(\mx_2)| \leq c \|\mx_1 - \mx_2\|$ for a constant $c$. A distribution $\calD$
with bounded domain and probability density function $f_{\calD}$ is near-uniform if $\frac{\sup f_{\calD}(\mx)}{\inf f_{\calD}(\mx)} = O(1)$.

\myparab{The forecasting problem.}
We operate in a time-dependent setting, where each timestamp $t$ can be construed as the $t^{th}$ round.
An interacting system consisting of $d$ entities (e.g., denoting stocks in the equity market or user accounts in a network), that are updated at each round, for a total number of $T$ rounds. Let $\my_{t, i} \in \reals$ denote the next-period forecast of entity $i$ at the $t$-th round, and $\my_{t} = (\my_{t,1}, \dots, \my_{t, d}) \in \reals^d$. Our goal is to forecast $\my_{t}$ based on all information available up to (but excluding) round $t$. 

\myparab{Model Assumptions.} Under the additive influence model, a generic model takes the form
\setlength{\belowdisplayskip}{0.5pt}
\begin{align} \label{eqn:model} 
\my_{t, i} = \sum_{j\leq d}\kappa(\mz_i,  \mz_j)g(\mx_{t,j}) + \xi_{t,i},
\end{align}
\vspace{-1mm}

and our goal is to learn $g(\cdot)$ and $\mz_i$'s with a total number of $n$ observations. Let $K \in \reals^{d \times d}$ such that $K_{i,j} = \kappa(\mz_{i}, \mz_{j})$. Here, we assume that  
\bb \emph{(A.1)} the vector representations $\mz_i$'s of the stocks and features $\mx_{t, i}$ are i.i.d. samples from (two different) near-uniform distributions on bounded supports, 
\bb \emph{(A.2)} $\mx_{t, i} \in [-1, 1]$ and $\E[g(\mx_{t, i})] = 0$, 
\bb \emph{(A.3)} $g(\cdot)$ is Lipschitz-continuous, and 
\mbox{\bb\emph{(A.4)}} $\xi_{t,i}$'s are zero-mean i.i.d. Gaussian random variables with standard deviation $\sigma_{\xi}$. 

We remark that (A.1) is standard in the literature~\cite{abraham2015low,sussman2013consistent,tang2013universally,li2017world,rastelli2016properties}. Assuming (A.2) simplifies the calculation and is without loss of generality, and (A.4) can also be relaxed to settings in which the $\xi_{t, i}$ variables are sub-Gaussian. See App.~\ref{asec:prelim} for a more detailed discussion of the assumptions.  

\section{Our algorithms}\label{sec:ourAlgos} 
\vspace{-1mm}
This section introduces our algorithmic pipeline in full detail. 
Sec.~\ref{k_est} describes an algorithm for learning the embedding without knowing $g(\cdot)$.  Sec.~\ref{sec:estimateg} explains the estimation of $g(\cdot)$ using machine learning techniques. Due to the space limit, detailed proofs  of all the Props 
are deferred to App.~\ref{asec:estk}.

\vspace{-1mm}
\subsection{Learning vector representation provably}\label{k_est}
\vspace{-1mm}
This section presents a provable algorithm to estimate the kernel matrix $K$ and the embedding $\mz_i$'s. Our algorithm does not require knowledge of $g(\cdot)$, thus providing a conceptually new approach to construct CEMs: high-dimensional learning of entity interactions can be decoupled from using ML techniques to fit the features. Because learning entity interactions could be a major source of causing overfitting,  disentangling it from the downstream task of learning $g(\cdot)$ enables us to leverage the function-fitting power of ML techniques without the cost of amplifying generalization errors.  
\begin{figure}[t!]
\vspace{-3mm}
    \centering
    \centering
  \subfigure[]{\label{fig1:a}\includegraphics[width=.22\textwidth]{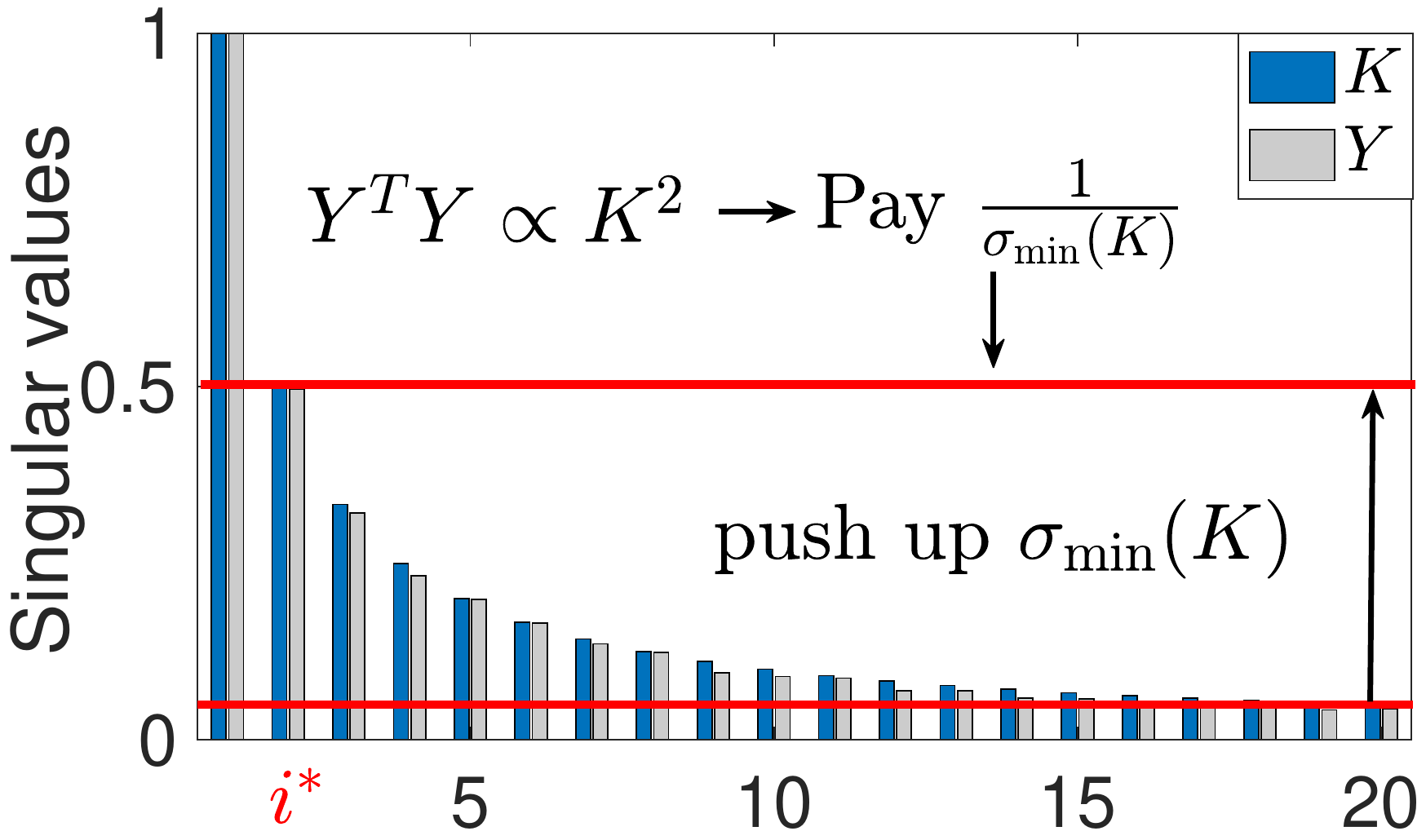}}
  \subfigure[]{\label{fig1:b}\includegraphics[width=.22\textwidth]{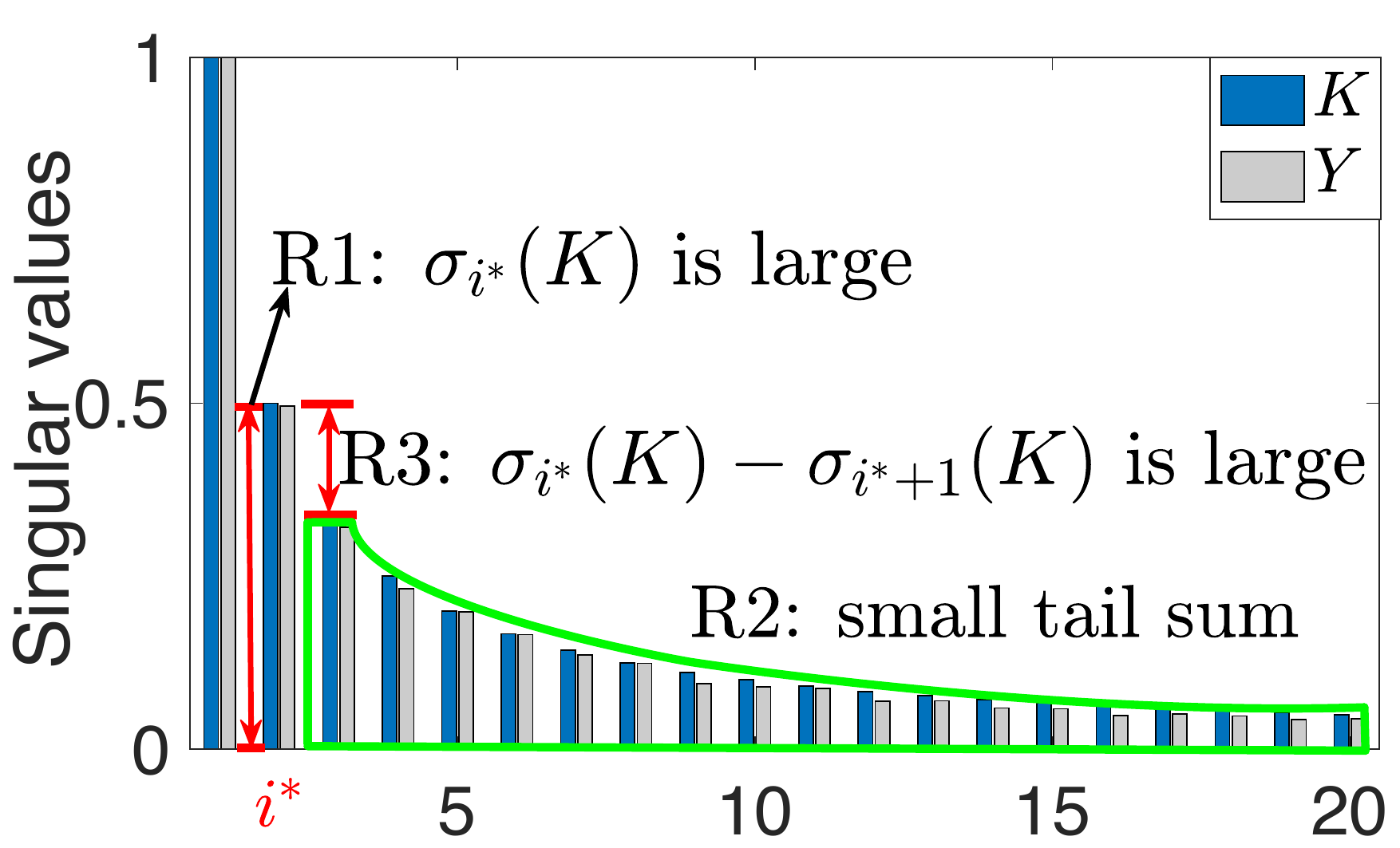}}
  \vspace{-4mm}
    \caption{\small{(a) We use the square root of $\calP_{i^*}(\mY^{\transpose}\mY)$ to approximate $K$ so that we pay a factor of $1/\sigma_{i^*}(K)$, instead of $1/\sigma_{\min}(K)$. (b) Three key requirements for $i^*$: \bb (R1) $\sigma_{i^*}(K)$ is large, \bb (R2) $\calP_{i^*}(K^2)$ is  close to $K^2$, and \bb (R3)  $\sigma_{i^*}(K) - \sigma_{i^*+1}(K)$ is large.}} 
    \label{fig:latentposalgo} 
    \vspace{-6mm}
\end{figure}

We next walk through our design intuition and start by introducing  additional notation. Let $\mY \in \reals^{n \times d}$ be such that $\mY_{t, i} = \my_{t, i}$ ($\mY$ 
is a matrix 
and $\my$ a random variable), $\mS \in \reals^{n \times d}$  with $\mS_{t, i} = \ms_{t, i} \triangleq g(\mx_{t,i})$, and $E \in \reals^{n \times d}$ with $E_{t, i} = \xi_{t, i}$. Recall that $K \in \reals^{d \times d}$ s.t. $K_{i,j} = \kappa(\mz_i, \mz_j)$, and $\calP_r(A)$ denotes $A$'s rank-$r$ approximation obtained by keeping the top $r$ singular values and vectors. 
Finally, for any PSD matrix $A$ with SVD $A = U\Sigma U^{\transpose}$, let $\sqrt{A} \triangleq U\Sigma^{\frac 1 2} U^{\transpose}$. 

Eq.~\eqref{eqn:model} can be re-written as $\mY = \mS K + E$, in which we need to infer $K$ using only $\mY$. We first observe that while none of the entries in $\mS$ are known, the $\mS_{t,i}$'s are i.i.d. random variables (because the $\mx_{t, i}$'s are i.i.d.);  therefore, our problem resembles a \textit{dictionary learning} problem, in which $K$ can be viewed as the dictionary to be learned, and $\mS$ is the measurement  matrix (see e.g.,~\cite{arora2014more}). However, in our case, $K$ is neither low-rank nor sparse, and we
cannot use standard dictionary learning techniques. 

First, we observe that, if infinitely many samples were available, then  $\mY^{\transpose}\mY/n$ approaches to $K^2$. 
Hence, intuitively we could use $\sqrt{\mY^{\transpose}\mY/n}$ to approximate $\sqrt{K^2} = K$. However, 
the existing standard matrix square root result has the notorious ``$1/\sigma_{\min}$-blowup'' problem, i.e., it gives us only $\|\sqrt{\mY^{\transpose}\mY/n} - KW\|_F \propto 1/\sigma_{\min}(K)$ ($W$ a unitary matrix), where typically 
$\sigma_{\min}(K)$ is extremely small, thus rendering the bound too loose to be useful~\cite{bhojanapalli2016dropping}. 

\begin{figure*}[t!]
\begin{minipage}[ht!]{0.52\linewidth}\vspace{-4mm}
    \centering
    \begin{algorithm}[H]\footnotesize
    \caption{nparam-gEST:}\label{fig:estimateg_main}
    \hspace*{\algorithmicindent} \textbf{Input} $\mX$, $\mY$, $\hat K$; \quad \\  
    \hspace*{\algorithmicindent} \textbf{Output} $\mu_1$ {\scriptsize (estimating other $\mu_i$'s is similar) }
    \hspace*{\algorithmicindent} 
    \begin{algorithmic}[1]
        \Procedure{nparam-gEST}{$\hat K, \mX, \mY$}
        \ForAll{$t \gets 1$ \To $n$}
        \State $q_t = \mathrm{Rand}(d)$
        \State  $ L_{(t, q_t), j} = \proc{Map-Regress}(q_t,\hat K, \mX_{t,:})$ 
        \EndFor
        \State \Return $\mu_1 \gets$ $\proc{FlipSign}$ $(q_t,\{\my_t, L_{(t, q_t), j}\}_{t \leq n})$
        \EndProcedure
        
        \Procedure{$\proc{Map-Regress}$}{$q_t, \hat K, \mx_t$} 
        \State Let $ L_{(t, q_t), j} = 0$ 
        \ForAll{$k \gets 1$ \To $d$}
        \State $ L_{(t, q_t), j} +=\hat  K_{q_t, k}$ with $j$ s.t. $\mx_{t, k} \in \Omega_j$.
        \EndFor
        \State \Return $L_{(t, q_t), j}$
        \EndProcedure
        
        \Procedure{$\proc{FlipSign}$}{$q_t, \{\my_t,  L_{(t, q_t), j}\}_{t \leq n}$} \label{filpsign}
        \ForAll{$t \gets 1$ \To $n$}
        \State $\hat\Pi^{(q_t)}_1(t) \triangleq  L_{(t,q_t), 1}  -  \frac{1}{\ell - 1}\left( \sum_{j \neq 1} L_{(t,q_t), j} \right)$
        \State $
        \tilde b_{t, q_t} = \left\{\begin{array}{ll}
        1 & \mbox{ if } \hat\Pi^{(q_t)}_1(t) \geq \frac{c}{\log d}\sqrt{\frac{d}{\ell}} \\
        -1 & \mbox{ if } \hat\Pi^{(q_t)}_1(t) <- \frac{c}{\log d}\sqrt{\frac{d}{\ell}} \\
        0 & \mbox{otherwise}
        \end{array}
        \right.$
        \EndFor 
        \State \Return{$\mu_1 = \frac{\sum_{t \leq n}\tilde b_{t, q_t}\my_{t, q_t}}{\sum_{t \leq n}\tilde b_{t, q_t}\hat\Pi^{(q_t)}_1(t)}$}
        \EndProcedure
    \end{algorithmic}
    \end{algorithm}
\end{minipage}
\hspace{2mm}
\begin{minipage}[ht!]{0.4\textwidth}
  \centering
  \includegraphics[width=0.98\linewidth]{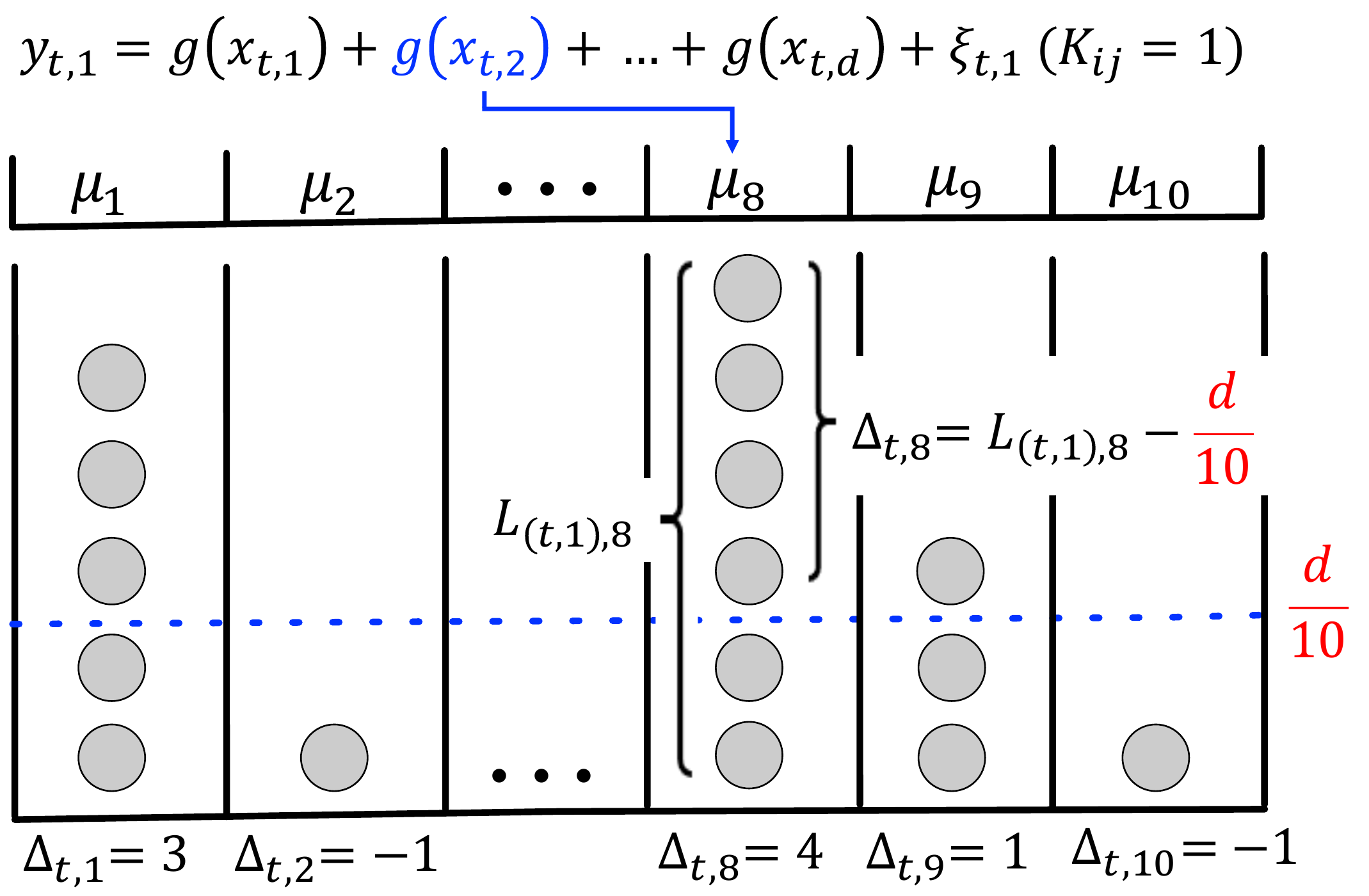}
 \vspace{1mm}
  \caption{\small{A toy example of nparam-gEST when $K_{i,j} = 1$ for all $i$ and $j$ and $\Omega = [-1,1]$ and is uniformly partitioned into 10 pieces. Sampling a $g(\mx_{t,i})$ corresponds to randomly placing a ball into a total number of 10 bins. For example, $\mx_{t,2}$ falls into the 8-th interval so $\mu_8$ is used to approximate $g(\mx_{t,2})$, which may be viewed as a new ball of type $\mu_8$ (or in 8-th bin) is created. The mean load for each bin is $d/\ell=d/10$. We calculate $\sum_{i \leq d}g(\mx_{t,i})$ by counting the balls in each bin: $\my_{t,1} =  5 \times \mu_1 + 1\times \mu_2 +...+6\times \mu_8 + 3\times \mu_9 + 1\times \mu_{10}+\xi_{t,1}$.}}
  \label{ballbins}
\end{minipage}
\vspace{-7mm}
\end{figure*}

To tackle the problem, our algorithm uses $\sqrt{\calP_{i^*}(\mY^{\transpose}\mY)/n}$ to approximate $K$ for a carefully chosen $i^*$ so that we pay a factor of $\sigma_{i^*}(K)$, instead of $\sigma_{\min}(K)$, to substantially tighten the error. See Alg.~\ref{fig:estimatek} in App.~\ref{asec:prop:main} and Fig.~\ref{fig1:a}. To implement this idea, we need to show that there always exists an $i^*$ such that
$\bullet$ \emph{(R1):}  $\sigma_{i^*}(K)$ is sufficiently large, 
 $\bullet$ \emph{(R2):} $\calP_{i^*}(K^2)$ is  close to $K^2$, and $\bullet$ \emph{(R3):} the spectral gap $\sigma_{i^*}(K) - \sigma_{i^*+1}(K)$ is sufficiently large so that we can use the Davis-Kahan theorem to prove that $\calP_{i^*}(K^2) \propto \calP_{i^*}(Y^{\transpose}Y)$~\cite{stewart1990matrix}. See also Fig.~\ref{fig1:b}. 
 
 These three requirements may not always be met simultaneously. For example, when $\sigma_i(K^2) \propto \frac 1 i$, the gap is insufficient and the tail diverges (R2 and R3 are violated). 
Therefore, we integrate the following two results.  
\bb\emph{(i)} The eigenvalues decay fast. This stems from two classical results from the \textit{kernel learning} literature. First, when $\kappa(\cdot, \cdot)$ is sufficiently smooth (such as the Gaussian, IMQ, or inner product kernels), the eigenvalues of the kernel operator $\calK$ associated with $\kappa(\cdot, \cdot)$ decay exponentially (e.g., $\lambda_i(\calK) \leq \exp(-C i^{\frac 1 r})$ for Gaussian kernels~\cite{belkin2018approximation}). Second, it holds true that $\sum_{i \geq 1}\left|\lambda_i(\calK) - \lambda_{i }(K/d)\right|^2_F \propto \frac 1 n$, a convergence result under the PAC setting~\cite{tang2013universally}. Therefore,  $\lambda_i(K)$ also approximately decays exponentially. \bb  \emph{(ii)} Combinatorial analysis between gaps and tails. We then leverage a recent analysis~\cite{wu2019adaptive} showing that when $\lambda_i(K)$ decays fast, it is always possible to find an $i^*$ such that $\lambda_{i^*}(\calK) - \lambda_{i^* + 1} (\calK)$ is sufficiently large (R1 \& R3 are satisfied) and $\sum_{j \geq i^*}\lambda^2_{j}(\calK) = o(1)$ (R2 is satisfied). Putting all these together leads to the following statement. 
%



\vspace{-0mm}
\begin{proposition}\label{prop:main}
Consider the additive influence model. Let $\kappa(\mz_i, \mz_j)$ be a Gaussian, inverse multi-quadratic (IMQ) or inner product kernel. 
Let $n \geq d$ be the number of observations and $\epsilon = \frac{c_0 \log^3 d}{\sqrt d}$. Assume that the noise level $\sigma_{\xi} = O(\sqrt{d})$. Let $\delta$ be a tunable parameter (also appeared in Alg.~\ref{fig:estimatek} in App.~\ref{asec:prop:main}) such that $\delta^3 = \omega(\epsilon^2)$.
There exists an efficient algorithm that outputs $\hat K$ such that  $\frac{1}{d^2}\|\hat K - K\|^2_F = O({\frac{\epsilon^2}{\delta^3}}+{\delta^{\frac 4 5}}) (= \tilde O(d^{-\Theta(1)}))$.  
\end{proposition}

%

We remark that \emph{(i)} the algorithm does not need to know the exact form of $\kappa$, so long as it is one of Gaussian, IMQ, or inner product kernels, \emph{(ii)} once $K$ is estimated, an Isomap-flavored algorithm may be used to estimate $\mz_i$'s~\cite{li2017world}, and \emph{(iii)} knowing $\hat K$ (without reconstructing $\mz_i$'s) is sufficient for the downstream $g(\cdot)$-learners.

\vspace{-1mm}
\subsection{Learning  $g(\cdot)$}\label{sec:estimateg}
\vspace{-1mm}
Here, we explain how prominent machine learning techniques, including neural nets (deep learning), non-parametric methods, and boosting, can be used to learn $g(\cdot)$. These techniques make different functional form assumptions of $g(\cdot)$,  and possess different ``iconic'' properties: deep learning assumes that $g(\cdot)$ can be represented by a possibly sophisticated neural net and uses stochastic gradient descent to train the model; non-parametric methods learn a Lipschitz-continuous $g(\cdot)$ with statistical guarantees; boosting consolidates forecasts produced from computationally efficient weak learners.  

Our setting has a different cost structure:
in univariate models, $g(\mx_{t, j})$ controls 
only one response 
$\hat \my_{t, j}$, but here, $g(\mx_{t, j})$ impacts all responses $\hat \my_{t, i}$,   
$i \in [d]$, as $\hat \my_{t, i} = \sum_{j} K_{i,j} g(\mx_{t,j})$. We 
generalize ML techniques  under
the new cost functions, while retaining the iconic properties of each technique. 

\myparab{Technique 1. Learn $g(\cdot)$ using neural nets.} When an estimate $\hat K$ is given, the training cost is $\sum_{t, i}(\my_{t, i} - \sum_{j \in [d]} \hat K_{i, j} g(\mx_{t, j}))^2$, in which case one can employ stochastic gradient descent when $g(\cdot)$ is a neural net.

\myparab{Technique 2. Learn $g(\cdot)$ using non-parametric methods.} When the response is univariate, e.g., $\my_{t, i} = g(\mx_{t,i})+\xi_{t,i}$, we can use a neighbor-based approach to estimate $g(\mx)$ for a new $\mx$: we identify one (or multiple) $\mx_{t, i}$'s in the training set that are close to the new $\mx$, and output $\my_{t,i}$ (or their averages, when multiple $\mx_{t, i}$ are chosen), using  $g(\mx) \approx g(\mx_{t,i})$, whenever $\mx$ is close to $\mx_{t, i}$. 

Here, we do not directly observe the values of individual $g(\mx_{t,i})$'s. Instead, each response is a linear combination of multiple $g(\cdot)$'s evaluated at different points, e.g., $\my_{t, 1} = K_{i, 1}\cdot g(\mx_{t, 1}) + \dots + K_{i, d} \cdot g(\mx_{t, d}) + \xi_{t,i}$. We show that finding neighbors reduces to solving a linear system. Furthermore, we design a moment-based algorithm, namely ``nparam-gEST'', which estimates $g(\cdot)$ with provable guarantees, as summarized in the following result.

\vspace{-1mm}
\begin{proposition}\label{prop:nonparam}
Consider the problem of learning an additive influence model with the same setup/parameters as in Prop.~\ref{prop:main}. Assume that $\mx_{t,i} \in \reals^{O(1)}$. Let $\ell$ be a tunable parameter. There exists an efficient algorithm to compute $\hat g(\cdot)$, based on $\hat K$ such that $\sup_{\mx} |\hat g(\mx) - g(\mx)| \leq (\log^6 n) \big(\sqrt{\gamma} + \sqrt{\frac{\ell}{n}} + \frac{1}{\ell}\big) = \tilde O(d^{-c})$ for suitable parameters, where $\gamma \triangleq {\frac{\epsilon^2}{\delta^3}}+{\delta^{\frac 4 5}}$.
\end{proposition}

\vspace{-1mm}
Our algorithm (Alg.~\ref{fig:estimateg_main}) consists of the following 3 steps: 

\vspace{-1mm}
\mypara{Step 1. Approximation of $g(\cdot)$.}  Partition $\Omega = [-1, 1]^k$ into subsets $\{\Omega_j\}_{j \leq \ell}$, and use piece-wise constant function to approximate $g(\cdot)$, i.e., $\tilde g(\mx_{t, i})$ takes the same value for all $\mx_{t, i}$ in the same $\Omega_j$. We partition $\{\Omega_j\}_{j\leq \ell}$ in a way such that $\Pr[\mx_{t, i} \in \Omega_j]$ are the same for all $j$. 

 \vspace{-1mm}
\mypara{Step 2. Reduction to linear regression.}
Each observation can be construed as a linear combination of $\mu_j$'s ($j\in [\ell]$), where 
$\mu_j = \E[g(\mx_{t, i}) \mid \mx_{t, i} \in \Omega_j]$.
For example,
$\my_{t, 1} = \sum_{i \leq d}K_{1, i} \mu_{j_i} + \xi_{t, 1}+o(1)$, where $\mx_{t, i} \in \Omega_{j_i}$, and in general, we have 

\vspace{-.5cm}
\setlength{\abovedisplayskip}{1pt} 
\setlength{\belowdisplayskip}{1pt}
\begin{align}
 \my_{t, i}  & = \sum_{j \leq \ell}L_{(t, i), j}\mu_j + \xi_{t, i} + o(1), 
\\ \nonumber 
\mbox{ where } \quad  L_{(t, i), j} & = \sum_{m \in \calL_{t, j}}K_{i, m} \mbox{ and } \calL_{t, j} = \{m: \mx_{t, m} \in \Omega_j\}.
\vspace{-7mm}
\end{align}\label{eqn:reduceregression}
\noindent Therefore, our learning problem reduces to a linear regression problem, in which the $L_{(t, i),j}$'s are features and the $\{\mu_j\}_{j \leq \ell}$ are coefficients to be learned.

\mypara{Step 3. Moment-based estimation.} 
An MSE-based estimator is consistent but finding its confidence interval (error bound) requires knowing the spectrum of the features' covariance matrix, which is remarkably difficult in our setting. Therefore, we propose a moment-based algorithm with provable performance (\proc{FlipSign}  in Alg.~\ref{fig:estimateg_main}). 


 We illustrate each steps above through a toy example, in which we assume $K_{i,j} = 1$ for all $i$ and $j$ so the model simplifies to $\my_{t, 1}   = \sum_{j \leq d} g(\mx_{t, j}) + \xi_{t,1}$. See Fig.~\ref{ballbins} for additional details. 

\paragraph{Steps 1 \& 2.}  
First, we view the generation of samples as a balls-and-bins process so that the $g(\cdot)$-estimation problem reduces to a regression problem (Steps 1 \& 2).
Specifically, we generate $(\my_{t, 1}, \{\mx_{t,i}\}_{i \leq d})$ as first sequentially sampling $\{\mx_{t,i}\}_{i \leq d}$ and computing the corresponding $g(\mx_{t, i})$, then summing each term up together with $\xi_{t,1}$ to produce $\my_{t,1}$. When an $\mx_{t,i}$ is sampled, it falls into one of $\Omega_i$'s with uniform probability. Let $j_i$ be the bin that $\mx_{t,i}$ falls into. Then $g(\mx_{t,i})$ is approximated by $\mu_{j_i}$ according to Step 1. Thus, we may view a ball of ``type $\mu_{j_i}$'' (or in $j_i$-th bin) is created. For example, in Fig.~\ref{ballbins}, $\mx_{t,2}$ falls into the 8-th interval so a ball is added in the 8-th bin. After all $\mx_{t,i}$'s are sampled, compute $\my_{t,1}$ by counting the numbers of balls in different bins. Recalling that the load of $j$-th bin is $L_{(t, 1), j}$, we have $\my_{t,1} \approx \sum_{j \leq d}L_{(t, 1), j}\cdot \mu_j + \xi_{t,1}$. Let $\Delta_{t,j} = L_{(t,1),j} - d/\ell$ and using that $\E[L_{(t,1),j}] = d/\ell$ and $\sum_{j \leq d}\mu_j = 0$, we have 
\begin{align}
\my_{t, 1}   = \Delta_{t,1} \mu_1 + \dots + \Delta_{t,\ell} \mu_{\ell} + \xi_{t, 1}.\label{eqn:reduceregress}
\end{align} 
Eq.~\eqref{eqn:reduceregress} is a standard (univariate) regression: for each $t$, we know $\my_{t,1}$, and know all $\Delta_{t,j}$'s because all $\mx_{t,j}$'s are observed so the number of balls in each bin can be calculated. We need to estimate the unknown $\mu_j$'s. Note that $\E[\Delta_{t,j}] = 0$. 

\vspace{-1mm}
\paragraph{Steps 3.} 
We solve the regression (Step 3). Our algorithm ``tweaks'' the observations so that the features associated with $\mu_1$ are always positive: let $b_{t, 1} = 1$ if $\Delta_{t,1} > 0 $ and $-1$ otherwise. Multiply $b_{t,1}$ to both sides of Eq.~\eqref{eqn:reduceregress} for each $t$,

\vspace{-2mm}
\begin{align}
b_{t, 1} \my_{t, 1}   
& = |\Delta_{t,1}| \mu_1 +  \dots + b_{t, 1} \cdot \Delta_{t,\ell} \cdot \mu_{\ell} + b_{t,1} \xi_{t, 1}. 
\label{eqn:sum}
\end{align} 

We sum up the LHS and RHS of (\ref{eqn:sum}) and obtain

\begin{align}\label{eqn:analysis}
\sum_{t \leq n} b_{t, 1}\my_{t, 1} 
= \big(\sum_{t\leq n}|\Delta_{t, 1}|\big)\mu_1 + \dots +\\  \nonumber    
\big(\sum_{t\leq n}b_{t,1}\cdot\Delta_{t, \ell}\big)\mu_\ell + \sum_{t\leq n} b_{t,1}\xi_{t, 1}.
\end{align}

Next, we have $\sum_{t\leq n}|\Delta_{t, 1}| = \Theta(n)$ whp. Also, we can see that $b_{t, 1}$ and $\Delta_{t,j}$ are ``roughly''
independent for $j \neq 1$ (careful analysis will make it rigorous). Therefore, for any $j\neq 1$, $\E[b_{t,1}\cdot \Delta_{t,j}] = 0$, and thus $\sum_{t\leq n}b_{t,1}\cdot\Delta_{t, j} = O(\sqrt{n})$ whp. Now (\ref{eqn:analysis}) becomes $\sum_{t \leq n} b_{t, 1}\cdot \my_{t, 1} = \big(\sum_{t}|\Delta_{t, 1}|\big)\mu_1 + O(\ell \cdot \sqrt{n})$. 
Thus our estimator is $\hat \mu_1 \triangleq \frac{\sum_{t } b_{t, 1}\cdot \my_{t, 1}}{\big(\sum_{t}|\Delta_{t, 1}|\big)} = \mu_1 + \frac{O(\ell\cdot \sqrt n)}{\Theta(n)} = \mu_1 + O\big(\frac{\ell}{\sqrt{n}}\big)$. Here, the covariance analysis  for the  $\Delta_{t,j}$'s is circumvented because $\Delta_{t,j}$'s interactions are compressed into the term $O\big(\frac{\ell}{\sqrt{n}}\big)$. We remark that the above analysis contains some crude steps and can be tightened up, as we have done in App.~\ref{implementfilpsign}.

\myparab{Technique 3. Learn $g(\cdot)$ using boosting.} In the univariate setting, we have $\my_{t, i} = \sum_{m \leq b}g_m(\mx_{t, i}) + \xi_{t, i}$, in which each $g_m(\mx_{t, i})$ is a weak learner. Standard boosting algorithms, such as~\cite{quinlan1986induction,chen2016xgboost},  assume that each $g_m(\cdot)$ is represented by a regression tree and constructed sequentially. A greedy strategy is used to build a new tree, e.g.,   iteratively splitting a node in a tree by choosing a variable that optimizes prediction improvement. In our setting, $\my_{t, i}$ depends on evaluating $g_m(\cdot)$ at $d$ different locations $\mx_{t, 1}, \dots, \mx_{t, d}$, so the splitting procedure either is $d$ (e.g. 3000 for equity market) times slower in a standard implementation, or requires excessive engineering tweak of existing systems.

Here, we propose a simple and effective weak learner based on the intuition of the tree structure. Let

{\small
\begin{align*}
(\mx_t)_i & = \big((\mx_{t, 1})_i, (\mx_{t, 2})_i, \dots, (\mx_{t, d})_i\big) \in \reals^d, \\
(\mx_t)_{i,j} & = \big((\mx_{t, 1})_i\cdot(\mx_{t, 1})_j, \dots, (\mx_{t, d})_i\cdot(\mx_{t, d})_j\big) \in \reals^d, \nonumber
\end{align*}
}

\noindent and $(\mx_{t})_{i,j,k}$ can be defined in a similar manner. We observe that regression trees used in GBRT models for equity return are usually shallow and can be \emph{linearized}: we may unfold a tree into disjunctive normal form (DNF)~\cite{abasi2014exact}, and approximate the DNF by a sum of multiple interaction terms, e.g., $I((\mx_{t, i})_1 > 0)\cdot I((\mx_{t, i})_2 > 0)$ can be approximated by $(\mx_{t, i})_1 \cdot (\mx_{t, i})_2$.

Our algorithm, namely ${\Linpvel}$ (linear projected vector ensemble learner), consists of weak learners in linear forms. Each linear learner consists of a subset of features and their interactions. The number of features included and the depth of their interactions are hyper-parameters corresponding to the depth of the decision tree. For example, if the first three features are included in the learner, we  need to fit $\my_{t, i}$ against   

\vspace{-2mm}
{\small
\begin{align}
  \sum_{j \in [d]} \underbrace{\hat K_{i,j}}_{\mbox{given}} \cdot \Big[ \underbrace{\beta_1 (\mx_{t, j})_1 + \dots}_{\mbox{linear terms}} \underbrace{+ \beta_4 (\mx_{t, j})_{1,2} + \dots + \beta_7 (\mx_{t, j})_{1,2,3}}_{\mbox{interaction terms}} \Big ],\label{fit_linear}  
\end{align}
}
\vspace{-3mm}

\noindent  by MSE.  Conceptually, although we use linearized models to approximate the trees, the ``target'' trees are unavailable (for the computational efficiency reasons above). We need a new procedure to select features for each learner. Our intuition is that,  if an interaction term could have predictive power, each feature involved in the interaction should also have predictive power. Our procedure is simply to select a fixed number of $i$'s with the largest $\mathrm{corr}((\my_{\mathrm{Res}})_t, \hat K (\mx_t)_i)$, where $(\my_{\mathrm{Res}})_t$ is the residual error. 




Using feature interactions to approximate DNF ($I((\mx_{t, i})_1 > 0)\cdot I((\mx_{t, i})_2 > 0) \approx (\mx_{t, i})_1 \cdot (\mx_{t, i})_2$) may not always be accurate, however, in our setting, linear interaction models often outperform decision trees or DNFs. We believe this occurs because interaction terms are continuous (whereas DNFs are  discrete functions), and thus they are more suitable to model smooth changes. 



\vspace{-1mm}
\section{Evaluation}\label{evluation}
\vspace{-1mm}
\begin{table*}[ht!]\vspace{-2mm}
\centering
\begin{adjustbox}{width=1.8\columnwidth,center}
\begin{tabular}{l|llrr|rrrr|rr} 
\hline
                                & \multicolumn{4}{l|}{\textit{Universe 800}}                                                                                              & \multicolumn{4}{l|}{\textit{Full universe}}                                                                          & \multicolumn{2}{l}{Backtesting}                       \\ 
\hline
Models                          & corr                                & w\_corr                             & \multicolumn{1}{l}{t-stat} & \multicolumn{1}{l|}{w\_t-stat} & \multicolumn{1}{l}{corr} & \multicolumn{1}{l}{w\_corr} & \multicolumn{1}{l}{t-stat} & \multicolumn{1}{l|}{w\_t-stat} & \multicolumn{1}{l}{PnL} & \multicolumn{1}{l}{Sharpe}  \\ 
\hline
Ours: \Linpvel                        & \multicolumn{1}{r}{\textbf{0.0764}} & \multicolumn{1}{r}{\textbf{0.0936}} & \textbf{6.7939}            & \textbf{6.3362}                & \textbf{0.0944}          & \textbf{0.1009}             & \textbf{8.2607}            & \textbf{6.4435}                & \textbf{0.5261}         & \textbf{10.97}              \\
Ours: nparam-gEST                     & \multicolumn{1}{r}{\textbf{0.0446}} & \multicolumn{1}{r}{\textbf{0.0320}} & \textbf{3.2961}            & \textbf{1.5753}                & \textbf{0.0618}          & \textbf{0.0553}             & \textbf{5.7327}            & \textbf{3.5212}                & \textbf{0.3386}         & \textbf{7.59}               \\
Ours: MLP                             & \multicolumn{1}{r}{\textbf{0.0550}} & \multicolumn{1}{r}{\textbf{0.0567}} & \textbf{6.4782}            & \textbf{5.0172}                & \textbf{0.0738}          & \textbf{0.0692}             & \textbf{9.2034}            & \textbf{6.4151}                & \textbf{0.4202}         & \textbf{9.43}               \\
Ours: LSTM                             & \multicolumn{1}{r}{\textbf{0.0286}} & \multicolumn{1}{r}{\textbf{0.0347}} & \textbf{3.4517}            & \textbf{3.0261}                & \textbf{0.0473}          & \textbf{0.0491}             & \textbf{6.3615}            & \textbf{4.2385}                & \textbf{0.2487}         & \textbf{7.10}               \\ 
\hline
\hline
UM: poor man Lin-PVEL           & 0.0674                              & 0.0866                              & 6.0947                     & 5.7312                         & 0.0827                   & 0.0884                      & 7.4297                     & 5.6659                         & 0.4565                  & 9.76                        \\
UM: poor man nparam-gEST        & 0.0432                              & 0.0309                              & 3.1505                     & 1.4912                         & 0.0584                   & 0.0509                      & 5.0098                     & 3.0844                         & 0.3070                  & 6.59                        \\
UM: MLP                         & 0.0507                              & 0.5050                              & 6.0234                     & 4.4966                         & 0.0606                   & 0.0467                      & 8.2857                     & 4.4555                         & 0.2782                  & 6.38                        \\
UM: LSTM                        & 0.0178                              & 0.0200                              & 2.2136                     & 1.8077                         & 0.0352                   & 0.0297                      & 4.0602                     & 2.3619                         & 0.175                   & 4.33                        \\
UM: Linear models     & 0.0106                              & 0.0192                              & 1.6471                     & 2.3030                         & 0.0290                   & 0.0251                      & 4.4711                     & 2.6010                         & 0.1888                  & 4.79                        \\
UM: GBRT                        & 0.0516                              & 0.0591                              & 7.5739                     & 5.6310                         & 0.0673                   & 0.0747                      & 9.3379                     & 7.8931                         & 0.3858                  & 4.45                        \\
UM: SFM                         & 0.0027                              & 0.0032                              & 0.4688                     & 0.4050                         & 0.0147                   & 0.0051                      & 1.2683                     & 0.3892                         & 0.0169                  & 0.54                        \\ 
\hline
Existing CEM: VR                & 0.0156                              & 0.0159                              & 2.4997                     & 1.7046                         & 0.0041                   & -0.0025                     & 0.8847                     & -0.3021                        & 0.0430                  & 1.20                        \\
Existing CEM: ARRR              & 0.0314                              & 0.0382                              & 2.5336                     & 2.4213                         & 0.0222                   & 0.0273                      & 1.8557                     & 1.8968                         & 0.1674                  & 3.24                        \\
Ad-hoc integration:  AlphaStock & 0.0085                              & 0.0063                              & 2.1045                     & 1.2516                         & 0.0027                   & 0.0032                      & 0.4688                     & 0.4050                         & 0.0045                  & 0.10                        \\
Ad-hoc integration: HAN         & 0.0105                              & 0.0081                              & 1.7992                     & 1.0017                         & 0.0080                   & 0.0050                      & 1.5716                     & 0.7340                         & 0.0570                  & 2.02                        \\ 
\hline
Consolidated: All Ours              & \textbf{0.0775}                     & \textbf{0.0950}                     & \textbf{6.8687}            & \textbf{6.4108}                & \textbf{0.0958}          & \textbf{0.1025}             & \textbf{8.5703}            & \textbf{6.6487}                & \textbf{0.5346}         & \textbf{11.30}              \\
\hline
\end{tabular}
\end{adjustbox}
\captionsetup{width=1.02\linewidth}
\captionsetup{font=small}
\vspace{-2mm}
\caption{Summary of results for equity raw return forecasts. \Linpvel $ $ is the gradient boosting method with the linear learner.
Boldface denotes the best performance in each group.  Backtesting results pertain to the  \textit{Full universe}.} 
\label{main_table}
\vspace{-7mm}
\end{table*}

We evaluate our algorithms on two real-world data sets: an equity market to predict stock returns, and a social network data set to predict user popularity,  respectively. Additional details and experiments for the equity market and Twitter data sets are in APP.~\ref{sec:app_exp}. We remark that this is a \emph{theoretical} paper; examining the performance on more data sets and baselines is a promising direction for future work. 

\vspace{-0mm}
\myparab{Models under our framework.}
We estimate $K$ and $g(\cdot)$ separately. To estimate $K$, we use both the algorithm discussed in Sec.~\ref{k_est} and other refinements discussed in App.~\ref{add_K}. To estimate $g(\cdot)$, we use SGD-based algorithms (MLP and LSTM), nparam-gEST, and \Linpvel. 


\vspace{-0mm}
\myparab{Baselines.} 
Our baselines include the commonly used models and domain specific models.
\emph{(i) The UMs} include linear, MLP, LSTM, GBRT, and SFM~\cite{zhang2017stock}. We also implement a ``poor man's version'' of both {\Linpvel} and nparam-gEST for UM, which assumes that influences from other entities are 0;   \emph{(ii) The CEMs} include a standard linear VAR~\cite{negahban2011estimation}, ARRR~\cite{wu2019adaptive}.
\emph{(iii) Ad-hoc integration}
AlphaStock~\cite{wang2019alphastock}, and HAN~\cite{hu2018listening} for the equity data set; Node2Vec~\cite{grover2016node2vec} for the Twitter data set.


\vspace{-1mm}
\myparab{Predicting equity returns.} We use 10 years of equity data from an emerging market to evaluate our algorithms and focus on predicting the next 5-day returns, for which the last three years are out-of-sample. The test period is substantially longer than those employed in recent   works~\cite{zhang2017stock,hu2018listening,li2019individualized}, adding to the robustness of our results.     
We constructed 337 standard technical factors to serve as a feature database for all models. 
We consider two universes: (i) \emph{Universe 800} can be construed as an equivalence to the S$\&$P 500 in the US, and consists of 800 stocks, and (ii) \emph{Full universe} consists of all stocks except for the very illiquid ones.  Visualizations are shown in App.~\ref{performance}.


We next describe our evaluation metrics and argue why they are more suitable and different from those employed in standard ML problems (see App.~\ref{robustevaluation}) 
$\bullet$ \emph{(i) Correlation vs MSE.} While the MSE is a standard metric for regression problems, correlations are better-suited  metrics for equity data sets~\cite{zhou2014active}. %
$\bullet$~\emph{(ii) Significance testing.} 
The use of  $t$-statistics estimators~\cite{newey1986simple} can account for the serial and cross-sectional correlations (App.~\ref{robustevaluation})
$\bullet$~\emph{(iii) Stock capacity/liquidity considerations.}  Predicting illiquid stocks is less valuable compared to predicting liquid ones because they cannot be used to build large portfolios. We use a standard approach to weight correlations (w\_corr) and $t$-statistics by a function of historical \textit{notional (dollar) traded volume} to reflect the capacity of the signals. 






 
 
\vspace{-1mm}
\mypara{Results.} See Table~\ref{main_table} for the results and 
the simulated Profit \& Loss (PnL). The experiments confirm that  \bb \emph{(i)} Models under our framework consistently outperform prior works. In addition, our {\Linpvel} model has the best performance;  \bb \emph{(ii)} By using a simple consolidation algorithm, the aggregated signal outperforms all individual ones. Our new models pick up signals that are orthogonal to existing ones because we rely on a new mechanism to use stock and feature interactions. 

\begin{table}[h!]
\centering
\begin{adjustbox}{width=1.\columnwidth,center}
\begin{tabular}{l|l|l|l|l} 
\hline
Models                   & MSE (in) & MSE (out) & Corr (in) & Corr (out)  \\ 
\hline
Ours: Lin-PVEL           & 0.472           & \textbf{0.520}      & 0.733            & \textbf{0.712}        \\
Ours: nparam-gEST        & 0.492           & \textbf{0.559}      & 0.688            & \textbf{0.658}        \\
Ours: MLP                & 0.486           & \textbf{0.547}      & 0.716            & \textbf{0.692}        \\
Ours: LSTM               & 0.484           & \textbf{0.541}      & 0.724            & \textbf{0.703}        \\ 
\hline
\hline
UM: Poor man Lin-PVEL    & 0.488           & 0.552               & 0.710            & 0.684                 \\
UM: Poor man nparam-gEST & 0.544           & 0.584               & 0.634            & 0.605                 \\
UM: Poor man MLP         & 0.506           & 0.562               & 0.703            & 0.673                 \\
UM: Poor man LSTM        & 0.496           & 0.559               & 0.710            & 0.679                 \\
UM: Linear models        & 0.616           & 0.663               & 0.618            & 0.592                 \\
UM: Random forest        & 0.611           & 0.659               & 0.623            & 0.587                 \\
UM: Xgboost              & 0.530           & 0.571               & 0.671            & 0.647                 \\ 
\hline
CEM: VR                  & 0.540           & 0.729               & 0.649            & 0.408                 \\
CEM: ARRR                & 0.564           & 0.652               & 0.610            & 0.573                 \\
Ad-hoc: Node2Vec         & 0.537           & 0.690               & 0.693            & 0.468                 \\ 
\hline
Consolidated: All Ours       & \textbf{0.459}  & \textbf{0.502}      & \textbf{0.767}   & \textbf{0.742}        \\
\hline
\end{tabular}
\end{adjustbox}
\vspace{-2mm}
\captionsetup{font=small}
\captionsetup{width=1.02\linewidth}
\caption{Overall in-sample and out-of-sample performance on the Twitter data set. Boldface denotes the best performance in each group. }
\label{table:twitter}
\vspace{-5mm}
\end{table}

\vspace{-1mm}
\myparab{Predicting user popularity in social networks.}
We use a Twitter data set to build models for predicting a user's \emph{next 1-day} popularity,  
defined as the sum of retweets, quotes, and replies received by the user.
We collected 15 months of Twitter data streams related to US politics. In total, there are 804 million tweets and 19 million distinct users. User $u$ has one interaction if and only if he or she is retweeted/replied/quoted by another user $v$. Due to the massive scale, we extract the subset of 2000 users with the most interactions, for evaluation purposes. For each user, we compute his/her daily popularity for 5 days prior to day $t$ as the features. 

\vspace{-1mm}
\mypara{Results.} 
We report the MSE and correlation for both \emph{in-sample} and \emph{out-of-sample} in Table~\ref{table:twitter}. We observe the consistent results with equity return experiments:
\emph{(i)} Methods under our framework achieve better performance in out-of-sample MSE and correlation, with  {\Linpvel}  attaining the overall best performance.
\emph{(ii)} Our methods yield the best generalization error by having a much smaller gap between training and test metrics.






\vspace{-2mm}
\section{Conclusion}
\vspace{-2mm}
This paper revisits the problem of building machine learning algorithms that involve interactions between entities. We propose an \textit{additive influence framework}  that  enables us to decouple the learning of the entity-interactions from the learning of feature-response interactions. Our upstream entity interaction learner has provable performance guarantees, whereas our downstream $g(\cdot)$-learners can leverage a wide set of effective ML techniques. All these methods under our framework are proven to be superior to the existing baselines.

%

\vspace{-1mm}
\section*{Acknowledgement}
\vspace{-2mm}
We thank anonymous reviewers for helpful comments and suggestions. Jian Li was supported in part  by the National Natural Science Foundation of China Grant 62161146004, Turing AI Institute of Nanjing and Xi'an Institute for Interdisciplinary Information Core Technology. Yanhua Li was supported in part by NSF grants IIS-1942680 (CAREER), CNS-1952085, CMMI- 1831140, and DGE-2021871. Zhenming Liu and Qiong Wu were supported by NSF grants NSF-2008557, NSF-1835821, and NSF-1755769.


\newpage

\newpage

\onecolumn

\tableofcontents
\appendix

\newpage

\section{Additional notes on problem definition}\label{asec:prelim}

\myparab{Independence of $\mx_{t, i}$.}  
Our analysis assumes that $\mx_{t, i}$ are independent across $t$'s and $i$'s. Our discussion assumes that $\mx_{t, i} \in \reals$. The arguments can easily generalize to multi-dimensional $\mx_{t, i}$. When $\mx_{t, i}$ are correlated across stocks, we can apply a factor model to obtain 
\begin{align}
    \mx_t = L \mf_t + \tilde \mx_t,  
\end{align}
where $\mx_t = (\mx_{t, 1}, \mx_{t, 2}, \dots, \mx_{t, d}) \in \reals^d$, $\mf_t$ is a low-dimensional vector that explains the co-moving (correlated) components, $L$ is the factor loading matrix, and $\tilde \mx_t = (\tilde \mx_{t, 1}, \dots, \tilde \mx_{t, d}) \in \reals^d$ is the idiosyncratic  component.  There exists a rich literature on algorithms that identify latent  factors~\cite{colby1988encyclopedia,fama1993common,hurst1965long,kakushadze2016101}. The shared factors driving the co-movements of the features can be utilized in other ways to forecast equity returns~\cite{ming2014stock}. We can use the idiosyncratic component $\tilde \mx_t$ as input features in our model, since the coordinates in $\tilde \mx_t$ are independent. In the setting where serial correlation is presented in $\tilde \mx_t$, one can use the standard differencing operator for decorrelating purposes~\cite{hamilton2000decorrelating}.


\section{Estimation of $K$}\label{asec:estk}

We prove Proposition~\ref{prop:main} and explain other variations of estimating $K$. 
For exposition purposes, our analysis focuses on the case where $\kappa$ is Gaussian kernel or IMQ. The case for $\kappa$ being an inner product function can be analyzed in a similar manner. See also Remark at the end of this section.

In Sec.~\ref{sec:est:background}, we first describe the background (e.g., notation and building blocks) needed. In Sec.~\ref{sec:est:background}, we present our proof for Prop~\ref{prop:main}. Our analysis assumes that $n \leq d^2$ to simplify calculations and ease the exposition. The case $n \geq d^2$ corresponds to the scenario when abundant samples are available, and is easier to analyze. In Sec.~\ref{asec:estk}, we explain additional algorithms for estimating $K$. 

\subsection{Background}\label{sec:est:background}
\myparab{Notation.} Let $A = \frac{1}{d^2}K^{\transpose}K$ and $B = \frac{1}{d^2 n}Y^{\transpose}Y$. Let $V^A_k$ be the first $k$ eigenvectors associated with $A$ and $V^B_k$ be the first $k$ eigenvectors associated with $B$. Note that $A$ and $B$ are symmetric. Let $\calP_A = V^A_{i^*}(V^A_{i^*})^{\transpose}$ and $\calP_B = V^B_{i^*}(V^B_{i^*})^{\transpose}$, where $i^*$ is defined in Alg.~\ref{fig:estimatek}.

\myparab{Distance between matrices.} 
For any positive-definite matrix $A$, there could be multiple square roots of $A$ (the square root is defined as any matrix $B$ such that $B B^{\transpose} = A$). Any pair of square roots of the same matrix differ only by a unitary matrix and should be considered as ``the same'' in most of our analysis. We adopt the following (standard) definition to measure the difference between two matrices. 

\begin{definition} (Distance between two matrices) Let $X, Y \in \reals^{d_1 \times d_2}$. The distance between $X$ and $Y$ is defined as 
\begin{equation}
    \dist^2(X, Y) = \min_{W \mbox{ unitary}} \|XW - Y\|^2_F. 
\end{equation} 
\end{definition}

\myparab{Building blocks related to distances.}

\begin{lemma}\label{lem:sqrt} (From~\cite{bhojanapalli2016dropping}) For any two rank-$r$ matrices $U$ and $X$, we have
\begin{align*}
    \dist^2(U, X) \leq \frac{1}{2(\sqrt 2 - 1) \sigma^2_r(X)}\|U U^{\transpose} - X X^{\transpose}\|^2_F. 
\end{align*}
\end{lemma}

\begin{lemma}\label{lem:M1M2} (From~\cite{ge2017no}) Let $M_1$ and $M_2$ be two matrices such that 
\begin{equation}
    M_1 = U_1 D_1 V^{\transpose}_1 \quad \mbox{ and } \quad  M_2 = U_2 D_2 V^{\transpose}_2. 
\end{equation}
It holds true that
\begin{equation}
    \|U_1 D_1 U^{\transpose}_1 - U_2 D_2 U^{\transpose}_2\|^2_F + \|V_1 D_1 V^{\transpose}_1 - V_2 D_2 V^{\transpose}_2\|^2_F \leq 2 \|M_1 - M_2\|^2_F. 
\end{equation}
\end{lemma}

\myparab{Building block related to gap vs. tail.}

\begin{lemma}\label{lem:gapvstail} Let $\{\lambda_i\}_{i \geq 1}$ be a sequence such that $\sum_{i \geq 1}\lambda_i = 1$, $\lambda_i \leq c i^{-\omega}$ for some constant $c$ and $\omega \geq 2$. Assume also that $\lambda_1 < 1$. Define $\delta_i = \lambda_i - \lambda_{i + 1}$, for $i \geq 1$. Let $\delta_0$ be a sufficiently small number, and $c_1$ and $c_2$ be two suitable constants. For any $\delta < \delta_0$, there exists an $i^*$ such that $\delta_{i^*} \geq \delta$ and $\sum_{j \geq i^*}\lambda_j =O\left(\delta^{\frac 4 5}\right)$. 
\end{lemma}

\myparab{Kernel learning.}
Let $\kappa(\mx, \mx')$ be a smooth radial basis function, i.e., $\kappa(\mx, \mx') = \kappa(\|\mx - \mx'\|)$, and use the notation $f(\cdot) = \kappa(\sqrt{\cdot})$. We assume that $|f^{(\ell)}(r)| \leq \ell! M^{\ell}$, for all $\ell$ sufficiently large and $r > 0$ . Note that both Gaussian kernels and inverse multi-quadratic kernels satisfy this property. 

Define an integral operator $\calK$ as 
\begin{equation}
    \calK f(\mx) =  \int \kappa(\mx, \mx') f(\mx') dF(\mx'), 
\end{equation}
where $F(\cdot)$ is the cumulative probability function over the support of $\mx$. Let $\calH$ be the rank of $\calK$, which can be either finite or countably infinite. 
Let $\psi_1, \psi_2, \dots, \psi_{\calH}$ be the eigenfunctions of $\calK$, and $\lambda_1, \lambda_2, \dots, \lambda_{\calH}$ be the corresponding eigenvalues such that $\lambda_1 \geq \lambda_2 \geq \dots $.  Let $K \in \reals^{d \times d}$ be the Gram matrix such that $K_{i,j} = \kappa(\|\mz_i - \mz_j\|)$.


Our analysis relies on the following two key building blocks. 

\begin{lemma}\label{lem:decay} (\cite{belkin2018approximation}) Let $\lambda^*_i$ be the $i$-th eigenvalue of $\calK$. There exist constants $C$ and $C'$ such that 
\begin{equation}
\lambda^*_i \leq C' \exp(-C i^{\frac 1 r}).     
\end{equation}
\end{lemma}

\begin{lemma}\label{lem:dk} 
Let $\lambda^*_i$ be the $i$-th eigenvalue of $\calK$. Let $\lambda_i(K)$ be the $i$-th eigenvalue of $K$. Let $\hat \lambda_j = \lambda_j(K)/d$. Let $\tau > 0$ be a tunable parameter. With probability at least $1 - \exp(-c_0 \tau)$ for some constant $c_0$, it holds true that  
\begin{equation}\label{eqn:dk}
    \left(\sum_{j \geq 1}(\lambda^*_j - \hat \lambda_j)^2\right)^{\frac 1 2} \leq 2 \sqrt{\frac{\tau}{d}}. 
\end{equation}
In addition, with probability at least 
 $1 - \exp(-c_0 \tau)$, 
\begin{equation}\label{eqn:dk2}
    \left(\sum_{j \geq 1}\left((\lambda^*_j)^2 - (\hat \lambda_j)^2\right)^2\right)^{\frac 1 2}\leq c \sqrt{\frac{\tau} d}
\end{equation}
for some constant $c$. 
\end{lemma}
\begin{proof}[Proof of Lemma~\ref{lem:dk}] Eq.~\ref{eqn:dk} is from Theorem B.2 from~\cite{tang2013universally}. Now to prove Eq.~\ref{eqn:dk2}, we have  
\begin{align*}
    & \left(\sum_{j \geq 1}\left((\lambda^*_j)^2 - (\hat \lambda_j)^2\right)^2\right)^{\frac 1 2} 
     = \left(\sum_{j \geq 1}(\lambda^*_j - \hat \lambda_j)^2(\lambda^*_j + \hat \lambda_j)^2\right)^{\frac 1 2} 
    \leq & c^{'} \left(\sum_{j \geq 1}(\lambda^*_j - \lambda_j)^2\right)^{\frac 1 2} 
    \leq & c \sqrt{\frac{\tau}{d}}. 
\end{align*}
\end{proof}

\subsection{Proof for Prop~\ref{prop:main}}\label{asec:prop:main}
\begin{minipage}[t!]{.9\linewidth}
    \centering
    \begin{algorithm}[H]\footnotesize
        \caption{Data-driven (DD) estimation of $K$}\label{fig:estimatek}
        \hspace*{\algorithmicindent} \textbf{Input} $\mX, \mY$;  \textbf{Output} $\hat K$
        \begin{algorithmic}[1]
        \State $[V, \Sigma, V^{\transpose}] = \mathrm{svd}\left(\frac 1 n\mY^{\transpose}\mY\right)$
        \State Let $\sigma_i = \Sigma_{i,i}$ 
        \State  $i^* = \max \{i:\sigma_{i^*} - \sigma_{i^*+1} \geq \delta d^2\}$
        \Comment{$\delta$ is tunable}
        \State \Return $\hat K = \calP_{i^*}(V\Sigma^{\frac 1 2}V^{\transpose})$
        \end{algorithmic}
    \end{algorithm}
\end{minipage}

Our analysis consists of four steps:

\begin{itemize}
    \item \textbf{Step 1.} Show that $\frac 1 n \mY^{\transpose}\mY - K^{\transpose}K$ is sufficiently small. 
    \item \textbf{Step 2.} Show that a low rank approximation of $\mY^{\transpose}\mY$ is sufficiently close to $\mY^{\transpose}\mY$. 
    \item \textbf{Step 3.} Show that $\calP_{i^*}(\frac 1 n\mY^{\transpose}\mY)$ is close to $\calP_{i^*}(\mK^{\transpose}\mK)$. 
    \item \textbf{Step 4.} Use results from the first three steps, together with Lemma~\ref{lem:gapvstail}, to prove the first part of Theorem~\ref{prop:main}. 
\end{itemize}

\myparab{Step 1. $\frac 1 n \mY^{\transpose}\mY$ and $K^{\transpose}K$  are close.} To formally prove this step, we rely on the following proposition.
\begin{proposition}\label{prop:yykk} Consider the problem of learning the stock latent embedding model. Let $n$ be the number of observations. Let $\mY \in \reals^{n \times d}$ be such that $\mY_{i, :}$ contains the $i$-th observation. Assume that $n \leq d^2$ and $\sigma_{\xi} = O(\sqrt{d})$. With overwhelming probability, it holds true that 
\begin{align}
    \left\|\frac 1 n \mY^{\transpose}\mY - K^{\transpose}K\right\|_F = O\left(\frac{d^2 \log^3 n}{\sqrt n}\right).
\end{align}
\end{proposition}

Proving Proposition~\ref{prop:yykk} requires a standard manipulation of concentration inequalities for matrices. 
See the proof in App.~\ref{sec:yykk}.

\myparab{Step 2. $\calP_{i^*}(K^{\transpose}K)$ is close to $K^{\transpose}K$.}


\begin{lemma} There exists a sufficiently large $d_0$ so that when $d \geq d_0$, Algorithm~\ref{fig:estimatek} always terminates. In addition, it holds true that  
\begin{align*}
    \sum_{i \geq i^*}\lambda^2_i\left(\frac K d\right) = O(\delta^{\frac 4 5}). 
\end{align*}
\end{lemma}

\begin{proof} Let $\tilde \delta = 10 \delta \geq \frac{c \log^3d}{\sqrt d}$ for a suitably large $c$. By Lemma~\ref{lem:gapvstail}, we have that there exists an $\tilde i$ such that
\begin{enumerate}
    \item $(\lambda^*_{\tilde i})^2 - (\lambda^*_{\tilde i + 1})^2 \geq \tilde \delta$. 
    \item $\sum_{i \geq \tilde i }(\lambda^*_i)^2 \leq \left(\tilde \delta\right)^{\frac 4 5}$.
\end{enumerate}

We first show that the algorithm terminates. We have that 
\begin{align*}
    \left|(\lambda^*_i)^2 - \lambda^2_i(K/d)\right| = O\left(|\lambda^*_i - \lambda_i(K/d)|\right) = O\left(\sqrt{\frac{\log d}{d}}\right)
\end{align*}
The last equality uses Proposition~\ref{prop:yykk}. 
Next, by using Lemma~\ref{lem:gapvstail}, we have 
\begin{align*}
    \left|\lambda_i\left(\frac 1 {n d^2}\mY^{\transpose}\mY\right) - \lambda_i\left(\frac{K^2}{d^2}\right)\right| = O\left(\frac{\log^3 n}{\sqrt n}\right). 
\end{align*}

Therefore, we can also see that 
\begin{align*}
    \lambda_{\tilde i}\left(\frac{1}{nd^2}\mY^{\transpose}\mY\right) - \lambda_{\tilde i + 1}\left(\frac{1}{nd^2}\mY^{\transpose}\mY\right)\geq \delta. 
\end{align*}
Our algorithm always terminates. In addition, we have $i^* \geq \tilde i$. Finally, we have 
\begin{align*}
    \sum_{i \geq i^*}\lambda^2_i\left(\frac{K}{d}\right) & \leq \sum_{i \geq \tilde i}\lambda^2_i\left(\frac Kd\right) \\
    & \leq 2 \left(\sum_{i \geq \tilde i}(\lambda^*_{\tilde i})^2 + \left(\lambda^*_{\tilde i} - \lambda_i\left(K/d\right)\right)^2\right) \\
    & = O\left({\tilde \delta}^{\frac 4 5}\right) = O\left(\delta^{\frac 4 5}\right). 
\end{align*}
\end{proof}

\myparab{Step 3. Analysis of the projection.} To show that 
$\calP_{i^*}\left(\frac 1 n \mY^{\transpose} \mY\right)$ and  $\calP_{i^*}(K^{\transpose}K)$ are close. We have the following lemma.

\begin{lemma}\label{lem:gap}
Consider running Algorithm $\proc{Estimate-K}$ in Alg.~\ref{fig:estimatek} for estimating $K$. Let $A = \frac 1 {d^2}K^{\transpose}K$ and $B = \frac 1 {d^2 n} \mY^{\transpose} \mY$. Let $\calP_A$ and $\calP_B$ be defined as above. 
Let $\epsilon \leq \frac{c_0 \log^3 d}{\sqrt d}$ for some constant $c_0$,  and 
$\delta$ be the gap parameter in Alg.~\ref{fig:estimatek} such that $\delta^3 = \omega(\epsilon^2)$. With high probability, we have
\begin{equation}
    \|\calP_A - \calP_B \|_2 = O\left(\frac{\|A - B\|_2}{\delta}\right). 
\end{equation}
\end{lemma}

\begin{proof}[Proof of Lemma~\ref{lem:gap}] Define $S_1 = [\lambda_{i^*}(A) - \delta/10, \infty)$ and $S_2 = [0, \lambda_{i^*}(A) + \delta / 10]$. By Lemma~\ref{prop:yykk}, we have $\left\|\frac{1}{nd^2} \mY^{\transpose}\mY - \frac{1}{d^2}K^{\transpose}K \right\|_F = O\left(\frac{\log^2 n}{\sqrt n}\right)$. Also using that $\delta \geq \frac{c \log^3 n}{\sqrt n}$, we have that $S_1$ contains the first $i^*$ eigenvalues of $A$ and $B$, whereas $S_2$ contains the rest of eigenvalues. We may then use a variant of the Davis-Kahan~\cite{stewart1990matrix} theorem to show that 

\begin{align*}
    \|\calP_A - \calP_B\|_2 \leq \frac{\|A - B\|_2}{0.8 \delta} \leq \frac{2\epsilon}{\delta}.
\end{align*}
\end{proof}

\myparab{Step 4. Gluing everything.} 
Recall that $A = \frac {1} {d^{2}} K^{\transpose}K$ and $B = \frac{1}{d^2 n}Y^{\transpose}Y$. 
Let $A_{i^*} = \calP_{A}(A) (= \calP_{i^*}(A))$ and $B_{i^*} = \calP_{i^*}(B)$. By Lemma~\ref{lem:gap}, we have

\begin{align*}
    \|A_{i^*} - B_{i^*}\|_F = & \|\calP_A(A) - \calp_B(B)\|_F, \\
    = & \|\calP_A(A) - \calP_B(A) + \calP_B(A) - \calP_B(B)\|_F, \\
    \leq & \|\calP_A - \calP_B\|_2 \|A\|_F + \|A- B\|_F, \\
    \leq & \Theta\left(\frac{\epsilon}{\delta} + \epsilon\right) = \Theta(\epsilon/\delta). 
\end{align*}

Next, we define the following matrix notation  
\begin{align*}
    A^{\frac 1 2}_{i^*} = U^A_{i^*}(\Sigma^A_{i^*})^{\frac 1 2} \quad \mbox{ and } \quad  B^{\frac 1 2}_{i^*} = U^B_{i^*}(\Sigma^B_{i^*})^{\frac 1 2}. 
\end{align*}

By Lemma~\ref{lem:sqrt}, there exists a unitary matrix $W$ such that
\begin{equation}
    \left\| U^A_{i^*}(\Sigma^A_{i^*})^{\frac 1 2}W - U^B_{i^*}(\Sigma^B_{i^*})^{\frac 1 2}\right\|^2_F = O\left(\frac{\epsilon^2}{\delta^3}\right). 
\end{equation}

By Lemma~\ref{lem:M1M2}, we obtain
\begin{align*}
    \|U^A_{i^*}(\Sigma^A_{i^*})^{\frac 1 2} U^A_{i^*} - U^B_{i^*}(\Sigma^B_{i^*})^{\frac 1 2}U^B_{i^*}\|^2_F = \left\| U^A_{i^*}(\Sigma^A_{i^*})^{\frac 1 2}W - \calP_{i^*}\left(\frac K d\right)\right\|^2_F = O\left(\frac{\epsilon^2}{\delta^3}\right). 
\end{align*}

Together with $\|\calP_{i^*}(K/d) - K/d\|^2_F = O\left(\delta^{\frac 4 5}\right)$, we have 

\begin{align*}
\left\|\frac{1}{d^2}\hat K - \frac 1{d^2}K\right\|_F & = \left\| U^A_{i^*}(\Sigma^A_{i^*})^{\frac 1 2}W - \frac K d\right\|_F = \left\| U^A_{i^*}(\Sigma^A_{i^*})^{\frac 1 2}W - \calP_{i^*}\left(\frac K d\right)\right\|_F + \left\|\calP_{i^*}\left(\frac K d\right) + \frac K d\right\|  \\
& = O\left(\sqrt{\frac{\epsilon^2}{\delta^3}}+\sqrt{\delta^{\frac 4 5}}\right).
\end{align*}
 
\myparab{Remark.} Our analysis relies only on the eigenvalues of $\calK$ decaying sufficiently fast. Many other kernels, such as inner product kernels with points on the surface of a unit ball~\cite{ha1986eigenvalues,azevedo2015eigenvalues}, also exhibit this property. In conclusion, our algorithms for estimating $K$ can be generalized to these $\kappa(\cdot, \cdot)$ functions. 


\subsection{Additional estimators for
$K$ \label{add_K}}

This section explains additional possible ways to estimate $K$. 
Our intuition is that estimations of $K$ effectively rely only on $\frac 1 n \mY^{\transpose}\mY$, which is the empirical covariance (aka risk) of the equities' returns/users' popularity. There are multiple ways to enhance the estimation of covariance matrix. Here, we focus on describing the estimation algorithm that is most effective in practice.
We estimate $K$ based on dynamically evolving hints.  Specifically, we assume that the latent positions evolve.  Let $K_t$ be the Gram matrix at round $t$.

\myparab{Construct $K_t$ from Twitter dataset} 
We faithful implementation of our algorithm discussed in Sec.~\ref{sec:ourAlgos}.
Moreover, we also estimate the evolving $K_t$ for the Twitter dataset by
maintain a sliding window $T$ (e.g. we use the data samples between $t -T$ and day $t$) and $T$ is a hyper-parameter to be tuned.

\myparab{Construct $K_t$ from Equity dataset.}   Because $K_t$ is evolving, we may not have sufficient data to track $\hat K$ in the market dataset. Therefore, we derive a new algorithm to estimate $\hat K$ using the so-called ``hint'' matrices, based on two observations: \emph{(i)} $\mY^{\transpose}\mY$ is effectively the covariance matrix of the returns. Third-party risk models such as Barra provide a more accurate estimation of the covariance matrix in practice. Thus, we may directly use the risk matrix produced by Barra as our estimation for $K$. \emph{(ii)} The movements of two stocks are related because they are economically linked. It is possible to estimate these links by using fundamental and news data. Specifically, we assume that $K_t = \exp(\beta_1 K^{(1)}_t + \dots + \beta_c K^{(c)}_t)$, where $K^{(i)}_t$ ($i \leq c$) can be observed. We then need only tune $\beta_i$'s to determine $\hat K$. Each of $K^{(i)}_t$ is considered as our ``hint''. We use a hint matrix $K^{(1)}_t$ constructed from Barra factor loading and a hint matrix $K^{(2)}$ constructed from news so that $\hat K_t = \exp(\beta_1 K^{(1)}_t + \beta_2 K^{(2)})$. The hint matrices are constructed as follows.

\mypara{$K^{(1)}_t$ from Barra loading.} Let $F_{t, i} \in \reals^{10}$ be the factor exposure of the $i$-th stock on day $t$. 
Construct $\hat K^{(1)}_t$ using two standard methods. 

\begin{itemize}
    \item \emph{Inner product.} $(\hat K^{(1)}_t)_{i,j} = \langle F_{t, i}, F_{t, j}\rangle$. 
    \item \emph{Distance.} $(\hat K^{(1)}_t)_{i,j} = \exp(- \lambda |F_{t, i} - F_{t, j}|^2)$, where $\lambda$ is a hyperparameter 
\end{itemize}

\mypara{$K^{(2)}_t$ from News Data.} 
We next build $K^{(2)}_t$ from the news using two steps. \emph{Step 1.} Construct $\tilde {K}^{(2)}_t \in \reals^{d \times d}$ such that  $(\tilde {K}^{(2)}_t)_{i,j}$ represents the number of news articles that mention both stock $i$ and stock $j$ between day $t - k$ and day $t$ (i.e., we maintain a sliding window of $k$ days and $k$ is a hyper parameter). \emph{Step 2.} Then construct $\hat K^{(2)}_t$ by taking a moving average of $\tilde { K}^{(2)}_t$.

\mypara{Construction of $\hat K_t$.} $\hat K_t$ be constructed from $\hat K^{(1)}_t$ (produced from Barra data) or $\hat K^{(2)}_t$ (constructed from news data set), or a consolidation of $\hat K^{(1)}_t$ and $\hat K^{(2)}_t$.
We shall examine the following consolidation algorithm. 
Specifically, we let $\hat K_t = \exp(\beta \hat K^{(1)}_t + (1-\beta) \hat K^{(2)}_t)$, where $\beta \in [0,1]$ and $\beta$ is a hyperparameter. 

\section{Estimating $g(\cdot)$ with non-parametric methods}\label{asec:estg}

This section proves Proposition~\ref{prop:nonparam}, i.e., we describe our non-parametric algorithm (nparam-gEST) for $\mx_{t, i} \in \reals^{O(i)}$ and analyze its performance. Assume that the probability cumulative density function $F_x(\cdot)$ of $\mx_{t, i}$ is known. In practice, this can be substituted by standard non-parametric density estimation methods~\cite{tsybakov2008introduction}.

We first describe a high-level roadmap of our algorithm analysis and then proceed to present the full analysis. 
\subsection{Overview of our algorithms}
As shown in Alg.~\ref{fig:estimateg_main}, our algorithm consists of three steps. 

\bb  \mypara{Step 1.} Partition the feature space $[-1, 1]^k$ into $\{\Omega_j\}_{j \leq \ell}$ so that $\Pr[\mx_{t, i} \in \Omega_j]$ are equal for all $j$. 

\bb \mypara{Step 2.} Reduce the original problem to a linear regression problem. 

\bb \mypara{Step 3.} Implement the \textit{FlipSign} algorithm for the scenario when only an estimated $\hat K$ available. 

We first comment on Steps 1 and 2. Then we explain the challenges in implementing the FlipSign idea, as well as our solution. 

\myparab{Step 1.} Construction of $\{\Omega_j\}_{j \leq \ell}$. We use a simple algorithm to find axis-parallel $\Omega_j$'s so that $\Pr[\mx_{t, i} \in \Omega_j]$ is uniform for all $j$. Recall that we assume that the cumulative probability function of $\mx_{t, i}$ is known (denoted as $F_{\mx}(\cdot)$). 

We describe the method for the case $k = 1, 2$ (recall that $k$ is the dimension of the feature $\mx_{t, i}$). Extensions to the case where $k \geq 3$ can be easily generalized. When $k = 1$, each $\Omega_j$ is simply an interval, and thus we only need to  find $\{ \mx_{1} = -1, \mx_2, \dots, \mx_{\ell+1} = 1 \}$ such that $F_{\mx}(\mx_{t+1}) - F_{\mx}(\mx_{t}) = 1 /\ell$ for all $1 \leq t \leq \ell$. For example, note that in the $k=1$ case, for $l=4$, the recovered values $\{\mx_1, \ldots, \mx_5 \}$ are simply identified with the usual quantiles of the distribution.   

\myparab{Step 2.} We next explain how the original problem can be reduced to a set of regression problems. Using MAP-REGRESS (line 8 in Alg.~\ref{fig:estimateg_main}). 
Recalling that for any $\my_{t, i} = \sum_{j \leq d} K_{i,j}g(\mx_{t,j}) + \xi_{t, i}$ (with fixed $i$ and $t$), we can approximate it as $\my_{t, i} = \sum_{j \leq d}K_{i, j}\tilde g(\mx_{t,j}) + \xi_{t,i}$. We may then re-arrange the terms and obtain

\begin{align}
   \my_{t, i} \approx \sum_{j \leq \ell}L_{(t, i), j}\mu_j + \xi_{t, i}, \mbox{ where } L_{(t, i), j} = \sum_{m \in \calL_{t, j}}K_{i, m} \mbox{ and } \calL_{t, j} = \{m: \mx_{t, m} \in \Omega_j\}.
\end{align}

Here, $\{\mu_j\}_{j \leq \ell}$ are unknown coefficients whereas $\my_{t, i}$ and $L_{(t, i), j}$ are observable.

Note that for any fixed $t$, there is a total number of $d$ observations (i.e., $\{(\my_{t, i}, \{L_{(t, i), j}\}_{i \leq d})\}_{i \leq d}$ and these observations are all correlated: $L_{(t, i), j}$ and $L_{(t, i'), j}$  depend on the same set $\calL_{t, j}$. So our algorithm chooses only one $i$ for each fixed $t$.

Sec~\ref{sec:estimateg} asserts that we can use the same $i$ for different $t$ when we have accurate information on $K$. In practice, we have only an estimate $\hat K$ of $K$. In addition, the estimation quality for any fixed $i$ depends on $\|K_{i, :} - \hat K_{i, :}\|^2_F$. We do not know a priori  which row of $\hat K$ is more accurate, although we know that on average, $\hat K$ is sufficiently close to $K$ (i.e., $\frac{1}{d^2}\|K -\hat K\|^2_F = o(1)$ from Proposition~\ref{prop:main}). To avoid the same ``bad'' $i$ being picked up repeatedly, we run a randomized procedure: let $q_t$ be a random number from $[d]$. We use the observations $\{(\my_{t, q_t}, \{\calL_{(t, q_t), 1}\}\}_{t \leq n}$ to learn the variables $\mu_1$. 

\subsection{Implementing the FlipSign algorithm} \label{implementfilpsign}
\myparab{Building a robust estimator.} We focus on estimating $\mu_1$ (See Line 15 in Algorithm~\ref{fig:estimateg_main}). The estimations for other $\mu_i$'s are the same. Let $b_{t, q_t}$ be the sign of $L_{(t, q_t), 1} - d/\ell$ but we only observe an estimate of $L_{(t, q_t), 1} $ (referred to as $\hat L_{(t, q_t), 1}$ in the forthcoming discussion). A major error source is that when $L_{(t, q_t), 1}$ gets too close to $d/\ell$, the sign of $\hat L_{(t, q_t), 1}$ 
 can be different from $L_{(t, q_t), 1}$ (i.e., $b_{t, q_t}$ is calculated incorrectly). We slove this problem by keeping only the observations when $|\hat L_{(t, q_t), 1} - d/\ell|$ is large. Specifically, let 
\begin{align}
    \Pi^{(q_t)}_1(t) \triangleq \sum_{k \in \calL_{t, 1}}K_{q_t, k} -\left( \sum_{k \notin \calL_{t, 1}}K_{q_t, k} \right)\frac{1}{\ell - 1}\label{eqn:Pidef}, 
\end{align}
and let $\hat \Pi^{(q_t)}_1(t)$ be computed using the estimate $\hat K$. We now define a robust variable $\tilde b_{t, q_t}$ to control the estimator

\begin{align}
 \tilde b_{t, q_t} = \left\{\begin{array}{ll}
        1 & \mbox{ if } \Pi^{(q_t)}_1(t) \geq \frac{c}{\log d}\sqrt{\frac{d}{\ell}} \\
        -1 & \mbox{ if } \Pi^{(q_t)}_1(t) <- \frac{c}{\log d}\sqrt{\frac{d}{\ell}} \\
        0 & \mbox{otherwise}.
        \end{array}\right.
\end{align}
\label{equation:b_qt}

In this case, the chance of obtaining an incorrect $\tilde b_{t, q_t}$ (i.e., $\hat \Pi^{q_t}_1(t) > \frac{c}{\log d}\sqrt{\frac d {\ell}}$ but $\Pi^{(q_t)}_1 \leq -\frac{c}{\log d}\sqrt{\frac{d}{\ell}}$ or vice verse) is significantly reduced (see line 19 in Alg.~\ref{fig:estimateg_main}).

\myparab{Analysis of the estimator.} 
Recall that $\{\Omega_j\}_{j \leq \ell}$ is a partition such that $\Pr[\mx_{t, i} \in \Omega_j]$ is uniform for all $j$. 
We first formalize the ``ideal'' $\mu_j$ that we want to track. Specifically, let $\mu_j = \E[g(\mx_{t, i}) \mid\mx_{t, i} \in \Omega_j]$. Our error analysis aims to track $\{\mu_j\}_{j \leq \ell}$ (i.e., we aim to find $(\hat \mu_j - \mu_j)^2$). We then articulate $\tilde g(\mx)$ as 
\begin{align*}
    \tilde g(\mx) = \mu_j, \mbox{ where } \mx \in \Omega_j.
\end{align*}

Our analysis consists of two parts. 

\mypara{Part 1. Analysis of a stylized model.} We analyze a model in which the observations are assumed to be generated from 
\begin{align}\label{eqn:tg}
    \my_{t, i} = \sum_{j \leq d}K_{i,j}\tilde g(\mx_{t, j}) + \xi_{t, i}, 
\end{align}
where $K_{i,j}$ is assumed to be known. 

Next, we analyze Alg.~\ref{fig:estimateg_main} when it is executed over this stylized model with the assumption that $K$ is given. 

\mypara{Part 2. Analysis of the original problem with $g(\cdot)$ and unknown $K$.} 
When we run Alg.~\ref{fig:estimateg_main} over the original process, we need to analyze two perturbations (deviations):
\begin{enumerate}
    \item $\my_{t, i}$ is generated through $g(\cdot)$, instead of $\tilde g(\cdot)$. 
    \item Our algorithm uses only an estimate of $K$. 
\end{enumerate}

\mypara{Warm-up and notation.} Before proceeding, let us introduce additional notation. Recall that  $\calL_{t, j} = \{k: \mx_{t, k} \in \Omega_j\}$, i.e., the set of $\mx_{t, k}$ that falls into the $j$-th bin $\Omega_j$ on time $t$. Also, recall that 
\begin{align*}
    L_{(t, i), j} = \sum_{k \in \calL_{t, j}}K_{i,k}.
\end{align*}

We have 
\begin{align*}
    \my_{t, i}& = \sum_{j \leq d} L_{(t, i), j} \mu_j + \xi_{t, i} = \left(\sum_{k \in \calL_{t, 1}}K_{i, k}\right)\mu_1 + \sum_{k \notin \calL_{t, 1}} K_{i,k} \tilde g(\mx_{t, k} \mid \mx_{t, k} \notin \Omega_1) + \xi_{t, i}.
\end{align*}

We interpret the meaning of the above equation. We treat $\mx_{t, k}$ and $\calL_{t, j}$ as random variables and the $\calL_{t, j}$'s are measurable by $\mx_{t, i}$. We imagine that an observation is generated by using the following procedure:

\begin{itemize}
    \item Step 1. Generate $\calL_{t, 1}$. That is, we determine the subset of ``balls'' (those $\mx_{t, i}$ for a fixed $t$) that fall into $\Omega_1$. 
    \item Step 2. Generate the rest of $\mx_{t, k}$ for $k \notin \calL_{t, 1}$ sequentially. 
    This corresponds to the terms 
$\sum_{k \notin \calL_{t, 1}} K_{i,k} \tilde g(\mx_{t, k} \mid \mx_{t, k} \notin \Omega_1)$. 
     $\mx_{t, k}$ is sampled from the conditional distribution $\mx_{t, k} \mid \mx_{t, k} \notin \Omega_1$. This explains why we write $\tilde g(\mx_{t, k} \mid \mx_{t, k} \notin \Omega_1)$. 
    \item Step 3. After $\calL_{t, 1}$ and $\mx_{t, k} \mid \mx_{t, k} \notin \Omega_1$ are fixed, we generate $\my_{t, i}$ using the stylized model. 
\end{itemize}

Let
\begin{align*}
    \tilde g_j(\mx_{t, i}) = \tilde g(\mx_{t, i}) - \E[\tilde g(\mx_{t, i}) \mid \mx_{t, i} \notin \Omega_j] = \tilde g(\mx_{t,i}) + \frac{\mu_j}{\ell - 1}. 
\end{align*}

We have
\begin{align}
    \my_{t,i} & = \underbrace{\left(\sum_{k \in \calL_{t, 1}} K_{i, k} - \left(\sum_{k \notin \calL_{t, 1}} K_{i,k}\right)\frac{1}{\ell - 1}\right)}_{\Pi^{(i)}_1(t)} \mu_1 + \underbrace{\sum_{k \in \calL_{t, 1}}K_{i,k} \tilde g_j (\mx_{t,i} \mid \mx_{t,i} \notin \Omega_1)}_{\Pi^{(i)}_2(t)} + \xi_{t,i} \label{eqn:fullrw}\\ & = \Pi^{(i)}_1(t) \mu_1 + \Pi^{(i)}_2(t) + \xi_{t, i} \nonumber.
\end{align}

We use the following abbreviation. 
\begin{itemize}
    \item $\hat b_t$ is an abbreviation for $\hat b_{t, q_t}$. 
    \item $\Pi_1(t)$ is an abbreviation for $\Pi^{(q_t)}_1(t)$. 
    \item $\Pi_2(t)$ is an abbreviation for 
    $\Pi^{(q_t)}_2(t)$. 
\end{itemize}

Let $\Delta = \hat K - K \in \reals^{d \times d}$ and $\Delta^2(q_t) = \sum_{i \leq d}\Delta^2_{q_t, i}$. Let $\hat \calB = \{ t\in [n]: \hat b_{t, q_t} = 1\}$.  Also, let \begin{align}
    s_t = \left\{
   \begin{array}{ll}
   1 & \mbox{ if } \Pi_1(t) > 0 \\
   -1 & \mbox{ otherwise.}
   \end{array}
    \right.\label{eqn:sdef}
\end{align}

\subsubsection{Part 1. Analysis of the stylized model}

Our main lemma in this section is an anti-concentration result on $\Pi^{(i)}_1(t)$ for any $i$ and $t$. 

\begin{lemma}\label{lem:singleanti} Let $\ell = O(d /\log^2 d)$. 
There exist constants $c_0$ and $c_1$ such that 
\begin{align*}
    \Pr\left[|\Pi^{(i)}_1(t)| \geq \frac{c_0}{\log d}\sqrt{\frac{d}{\ell}}\right] \geq c_1
\end{align*}
The probability is over the random tosses of $\{\mx_{t, i}\}_{i \leq d}$. 
\end{lemma}

\begin{proof} We use a random-walk interpretation of $\Pi^{(i)}_1$. For each $k \in [d]$, with probability $1/\ell$, it (i.e., $\mx_{t, k}$) falls into $\calL_{t, j}$. When this happens, $\Pi^{(i)}_1(t)$ is incremented by $K_{i,k}$. With probability $1-1/\ell$, it does not fall into $\calL_{t, j}$. In this case, $\Pi^{(i)}_1(t)$ is decremented by $K_{i,t}/(\ell-1)$. 

We define a sequence $\{Z_k\}_{k \leq d}$ to clarify the random-walk interpretation. 
\begin{align*}
    Z_k = \left\{
    \begin{array}{ll}
    K_{i,k} & \mbox{with probability }\frac 1 \ell.  \\
    -\frac{K_{i,k}}{\ell-1} & \mbox{with probability } 1-\frac 1\ell.   
    \end{array}
    \right.
\end{align*}

We couple $\Pi^{(i)}_1(t)$ with $\{Z_i\}_{i \leq d}$ such that $\Pi^{(i)}_1(t) = \sum_{k \leq d} Z_k$. Apply Lemma~\ref{lem:anti} (a folklore that generalizes Littlewood-Offord-Erd\H{o}s) to prove our Lemma. 
\end{proof}

\subsubsection{Part 2. Analysis of the original problem with $g(\cdot)$ and unknown $K$}

Our analysis consists of three components.

\mypara{Part 2.1. Building blocks.} We develop the essential building blocks needed in our analysis. 

\mypara{Part 2.2. Using $\hat K$.} We show that when $K$ is substituted by $\hat K$, the error of the estimator is well-managed. 

\mypara{Part 2.3. Using $g(\cdot)$.} We show that when $\tilde g(\cdot)$ is substituted by $g(\cdot)$, not much additional error is introduced. 

\myparab{Part 2.1. Building blocks.} We start with a variance-based Chernoff bound~\cite{chung2006concentration}. 

\begin{theorem}\label{thm:chungchernoff} Suppose that $X_i$ are independent random variables satisfying $X_i \leq M$ for $1 \leq i \leq n$. Let $X = \sum_{i = 1}^n X_i$ and $\|X\| = \sqrt{\sum_{i = 1}^n \E[X^2_i]}$. Then we have 
\begin{equation}
    \Pr[X \geq \E[X] + \lambda] \leq \exp\left(-\frac{\lambda^2}{2(\|X\|^2 + M \lambda /3}\right).
\end{equation}
\end{theorem}

\begin{lemma}\label{lem:absdiff} Consider running Alg.~\ref{fig:estimateg_main} to learn the stylized model. Let $\hat K$ be such that $|\hat K - K|^2_F \leq \gamma d^2$, for $\gamma = o(1)$. Let $\Pi^{(i)}_1(t)$ and $\hat \Pi^{(i)}_1(t)$ be those defined around Eq.~\ref{eqn:Pidef}. 
Let $\Delta = \hat K - K \in \reals^{d \times d}$ and $\Delta^2(q_t) = \sum_{i \leq d}\Delta^2_{q_t, t}$. Let $\lambda_t$ be any random variable that is measurable by $q_t$. 
With high probability we have
\begin{align*}
    \Pr\left[\left|\hat \Pi^{(q_t)}_1(t) - \Pi^{(q_t)}_1(t)\right| \geq \lambda_t \mid q_t\right] \leq \exp\left(-\frac{\lambda^2_t}{\Delta^2(q_t)/\ell + \lambda_t/3}\right),  
\end{align*}
and 
\begin{align*}
    \sum_{t \leq n}\left|\hat \Pi^{(q_t)}_1(t) - \Pi^{(q_t)}_1(t)\right| = O\left(n \log n\sqrt{\frac{\gamma d}{\ell}}\right). 
\end{align*}
\end{lemma}

\begin{proof}  We shall again use random-walk techniques to analyze $\left|\hat \Pi^{(q_t)}_i(t) - \Pi^{(q_t)}_i(t)\right|$. Let 
\begin{align*}
    Z_i = \left\{
    \begin{array}{ll}
    \Delta_{q_t, i} &\mbox{with probability} \frac{1}{\ell}  \\
    -\frac{\Delta_{q_t, i}}{\ell - 1}     & \mbox{with probability } 1 - \frac{1}\ell. 
    \end{array}
    \right.
\end{align*}
We have $\E[Z^2_i] = O\left(\frac{\Delta^2_{q_t, i}}{\ell}\right)$, which implies that $\sqrt{\sum_{i \leq d}\E[Z^2_i]} = \sqrt{\frac{\sum_{i \leq d}\Delta^2_{q_t, i}}{\ell}}$. Also, we can use a standard way to couple $Z_i$'s with $\hat \Pi_1^{(q_t)}(t)$ and $\Pi_1^{(q_t)}(t)$ such that 
\begin{align*}
    \left|\sum_{i \leq d}Z_i\right| = \left|\hat \Pi^{(q_t)}_i(t) - \Pi^{(q_t)}_i(t)\right|.
\end{align*}

By using a Chernoff bound from Theorem~\ref{thm:chungchernoff}, we have 
\begin{align*}
    \Pr\left[\left|\sum_{i \leq d}Z_i\right|\geq \lambda_t \mid q_t\right] \leq \exp\left(- \frac{\lambda^2_t}{\frac 1{\ell}\left(\sum_{i \leq d}\Delta^2_{q_t,i}\right)+\frac{\lambda}{3}}\right) = \exp\left(-\frac{\lambda^2_t}{\Delta^2(q_t)/\ell + \lambda_t/3}\right). 
\end{align*}
This proves the first part of the Lemma. Next, we set $\lambda_t = c_0 (\log d)\sqrt{\frac{\Delta^2(q_t)}{\ell}}$. Then we obtain $ \Pr\left[\left|\sum_{i \leq d}Z_i\right|\geq \lambda_t \mid q_t\right] = \exp(-\Theta(\log^2 d))$. 
Now, conditioned on knowing $\{q_t\}_{t \leq n}$, with high probability, we have 
\begin{align*}
    \sum_{t \leq n}\left|\hat \Pi^{(i)}_1(t) - \Pi^{(i)}_1(t)\right| \leq \sum_{t \leq n}\lambda_t = \frac{\log d}{\sqrt{\ell}}\sum_{t \leq n}\sqrt{\Delta^2(q_t)}. 
\end{align*}
Next, we give a concentration bound for $\sum_{t \leq n}\sqrt{\Delta^2(q_t)}$. Let $v^2_i = \|K_{i, :} - \hat K_{i, :}\|^2$. We know that $\Delta^2(q_t)$ can only take values from $v_1^2, \dots , v^2_d$ with $\sum_{i \leq d}v^2_i = \|\hat K - K\|^2_F \leq \gamma d^2$. We have $\sqrt{\Delta^2(q_t)} \leq \sqrt{\gamma}d$. 

Again using the condition that $\|\hat K - K\|^2_2 \leq \gamma d^2$  and Jensen's inequality, we have $\E[\sqrt{\Delta^2(q_t)}] \leq \sqrt{\gamma d}$. 

Use the Chernoff bound, we have  
\begin{align*}
    \Pr\left[\sum_{t \leq n}\sqrt{\Delta^2(q_t)} \geq \E\left[\sum_{t \leq n}\sqrt{\Delta^2(q_t)}\right]+\lambda\right] \leq \exp\left(-\frac{\lambda^2}{\gamma d n + \sqrt{\gamma d \lambda}/3}\right). 
\end{align*}
We set $\lambda = \frac{\sqrt{\gamma d n}}{\log^2d}$ such that the right hand side is negligible. Now with high probability we have
\begin{align*}
    \sum_{t \leq n}\left|\hat \Pi^{(i)}_1(t) - \hat \Pi^{(i)}_1(t)\right| \leq \frac{\log d}{\sqrt {\ell}} \sum_{t \leq n}\sqrt{\Delta^2(q_t)} = O\left(n \log n\sqrt{\frac{\gamma d}{\ell}}\right).
\end{align*}

\end{proof}

\begin{fact} For any $j$, 
\begin{align*}
    \E[g(\mx_{t, k})\mid \mx_{t,k} \notin \Omega_j] & = \E[\tilde g(\mx_{t, k}) \mid \mx_{t,k} \notin \Omega_j]  = -\frac{\mu_j}{\ell-1}
\end{align*}
\end{fact}
\begin{proof}
By our model assumption,
\begin{equation}
    \E[\ms_{t, k}] = \E[g(\mx_{t, k})] = \mu_1 + \dots + \mu_{\ell} = 0. 
\end{equation}
On the other hand, 
\begin{align*}
    \E[g(\mx_{t,k}) \mid \mx_{t, k} \notin \Omega_j] = \frac{\mu_1 + \dots + \mu_{j - 1} + \mu_{j+1}+ \dots + \mu_{\ell}}{\ell - 1} = - \frac{\mu_j}{\ell - 1}. 
\end{align*}
Similarly, we prove that 
$\E[g(\mx_{t, k})\mid \mx_{t,k} \notin \Omega_j] = \E[\tilde g(\mx_{t, k}) \mid \mx_{t,k} \notin \Omega_j]$.
\end{proof}

\begin{lemma}\label{lem:blarge} Let $\hat \calB = \{t \in [n]: \hat b_{t, q_t} = 0\}$, where $\hat b_{t, q_t}$ is defined in Sec~\ref{equation:b_qt}. With high probability we have $|\hat \calB| = \Omega(n)$. 
\end{lemma}

This can be shown by Lemma~\ref{lem:singleanti} and a Chernoff bound.

\begin{lemma}\label{lem:pi2small} Let $s_t$ be defined in Eq.~\ref{eqn:sdef}. Recall that $\hat \calB = \{ t\in [n]: \hat b_{t, q_t} = 1\}$. We have
\begin{align*}
    \sum_{t \in \hat \calB}\Pi_2(t) s_t  = \sum_{t \in \hat \calB} s_t \left(\sum_{k \notin \calL_{t, 1}}K_{q_t, k}\tilde g_1(\mx_{t, k} \mid \mx_{t, k} \notin \Omega_1)\right)  = O(\sqrt{n d}). 
\end{align*}
\end{lemma}
\begin{proof} Our key observation is that conditioned on $t \in \hat \calB$ and $\mx_{t, k} \notin \Omega_1$, $\tilde g_1(\mx_{t, k})$'s are bounded independent zero-mean random variables. Also $|\hat \calB| = \Omega(n)$ (Lemma~\ref{lem:singleanti}). Therefore, a standard Chernoff bound gives $\sum_{t \in \hat \calB}s_t \Pi_2(t) = O(\sqrt{nd})$. 
\end{proof}

\begin{lemma}\label{lem:bpi} Recall that $\hat \calB = \{ t\in [n]: \hat b_{t, q_t} = 1\}$, we have 
\begin{align*}
    \left|\sum_{t \in \hat \calB} \hat b_t \Pi_1(t) - \sum_{t \in \hat \calB}|\Pi_1(t)|\right| \leq 2n \log^5 d \sqrt{\frac{d}{\ell}}\gamma
\end{align*}
\end{lemma}

\begin{proof} Recall that  
\begin{align*}
    s_t = \left\{
   \begin{array}{ll}
   1 & \mbox{ if } \Pi_1(t) > 0 \\
   -1 & \mbox{ otherwise.}
   \end{array}
    \right.
\end{align*}

We have 
\begin{align*}
    \left\|\sum_{t \in \hat \calB} \hat b_t \Pi_1(t) - \sum_{t \in \hat \calB}|\Pi_1(t)|\right\| \leq 2 \sum_{t \in \hat \calB} |\Pi_1(t) | I(\hat b_t \neq s_t),  
\end{align*}
where $I(\cdot)$ is an indicator function that sets to 1 if and only if its argument evaluates to true. 
Note that $I(\hat b_t \neq s_t)$'s are i.i.d. random variables for different $t$'s. We compute 

\begin{align*}
    & \Pr[I(\hat b_t \neq s_t)] \\
    \leq & \Pr[s_t = -1 \wedge \hat b_t = 1] + \Pr[s_t = 1 \wedge \hat b_t = -1] \\
    = & \Pr\left[\Pi_1(t) < 0 \wedge \hat \Pi_1(t) > \frac{c_0}{\log d}\sqrt{\frac d {\ell}}\right] + \Pr\left[\Pi_1(t) > 0 \wedge \hat \Pi_1(t) < - \frac{c_0}{\log d}\sqrt{\frac d {\ell}}\right]\\
    \leq & \Pr\left[\left|\Pi_1(t) - \hat \Pi_1(t)\right| > \frac{c_0}{\log d}\sqrt{\frac d {\ell}}\right] \\
     \leq & \Pr\left[\left(|\Pi_1(t) - \hat \Pi_1(t)| > \log d \sqrt{\frac{\Delta^2(q_t)}{\ell}}\right)\vee \left(\log d \cdot \sqrt{\frac{\Delta^2(q_t)}{\ell}}\geq \sqrt{\frac{d}{\ell}}\frac{c_0}{\log d}\right) \right] \\
     \leq & \frac{1}{n^{10}} + \Pr\left[\log d \cdot\sqrt{\frac{\Delta^2(q_t)}{\ell}}\geq \sqrt{\frac{d}{\ell}}\frac{c_0}{\log d}\right]\quad \mbox{ (by Lemma~\ref{lem:absdiff})} \\
  = & \Pr\left[(\log^4 d) \Delta^2(q_t) \geq d\right] + n^{-10} . 
\end{align*}
Note that $\E[\Delta^2(q_t)] = \gamma d$. Using a Markov inequality, we have 
\begin{equation}
    \Pr\left[\log^4 d \Delta^2(q_t) > d\right] \leq \gamma \log^4 d.
\end{equation}

Using the fact that $I(\hat b_t \neq s_t)$ are independent across $t$, with high probability we have 

\begin{align*}
    \left|\sum_{t \in \hat \calB} \hat b_t \Pi_1(t) - \sum_{t \in \hat \calB}|\Pi_1(t)|\right| \leq 2n \log^5 d \sqrt{\frac{d}{\ell}}\gamma
\end{align*}
\end{proof}

\myparab{Part 2.2. When $K$ is substituted by $\hat K$.} When $K$ is substituted by $\hat K$, our estimator becomes

\begin{align*}
    \hat \mu_1 &  = \frac{\left(\sum_{t \in \hat \calB}\hat b_t \Pi_1(t)\right) + \sum_{t \in \hat \calB}\left(\hat b_t \Pi_2(t) + \hat b_t \xi_{t,q_t}\right)}{\sum_{t \in \hat \calB}\hat b_t \hat \Pi_1(t).} \\
    & = \frac{\sum_{t \in \hat \calB}\hat b_t \Pi_1(t)}{\sum_{t \in \hat \calB}\hat \Pi_1(t)} \mu_1 + \frac{\sum_{t \in \hat \calB}\left(\hat b_t \Pi_2(t) + \hat b_t \xi_{t, q_t}\right)}{\sum_{t \in \hat \calB}\hat b_t \hat \Pi_1(t)}.
\end{align*}

We note that 
\begin{align*}
    &\left|\sum_{t \in \hat \calB} \hat b_t \hat \Pi_1(t) - \sum_{t \in \hat \calB}\hat b_t \Pi_1(t)\right| 
    \leq & \sum_{t \in \hat \calB}\left|\hat \Pi_1(t) - \Pi_1(t)\right| 
    \leq & \sum_{t \leq n}\left|\hat \Pi_1(t) - \Pi_1(t)\right|
    \leq & n \sqrt{\frac{\gamma d}{\ell}} \log d \quad \quad \mbox{(Lemma~\ref{lem:absdiff}).}
\end{align*}

Also, we can see that (Lemma~\ref{lem:bpi}) 
\begin{align*}
    \left|\sum_{t \in \hat \calB}\hat b_t \Pi_1(t) - \sum_{t \in \hat \calB}|\Pi_1(t)|\right| \leq c_0 n \log^5 d \sqrt{\frac{d}{\ell}}\gamma.
\end{align*}

Both of the inequalities above imply that with high probability, the following holds true
\begin{align*}
    \left|\sum_{t \in \hat \calB}\hat b_t \hat \Pi_1(t) - \sum_{t \in \hat \calB}|\Pi_1(t)|\right| = O\left(n \log^5 d \sqrt{\frac{d}{\ell}}\sqrt{\gamma}\right).
\end{align*}
Next, by Lemma~\ref{lem:blarge} and Lemma~\ref{lem:singleanti}, we have 
\begin{align*}
    \sum_{t \in \hat \calB}|\Pi_1(t)| = \Omega\left(\frac{n}{\log n}\sqrt{\frac{\ell}{d}}\right). 
\end{align*}

Therefore, 
\begin{align*}
    \sum_{t \in \hat \calB}\hat b_t\Pi_1(t) & = (1+\tau_1) \sum_{t \in \hat \calB}|\Pi_1(t)|  \\
    \sum_{t \in \hat \calB}\hat b_t \hat \Pi_1(t) & = (1+\tau_2)\sum_{t \in \hat \calB}|\Pi_1(t)|, 
\end{align*}

where $|\tau_1|, |\tau_2| = O\left(\log^6 n \sqrt{\gamma}\right)$. This implies that 
\begin{align*}
\left|    \frac{\sum_{t \in \hat \calB}\hat b_t \Pi_1(t)}{\sum_{t \in \hat \calB}\hat b_t \hat \Pi_1(t)}-1\right|  = O(\tau_1). 
\end{align*}

Now we analyze the second term. By Lemma~\ref{lem:blarge} and Lemma~\ref{lem:singleanti}, we have $\sum_{t \in \hat \calB}\hat b_t \hat \Pi_1(t) = \Omega\left(\frac{n}{\log d}\sqrt{\frac{d}{\ell}}\right)$. By Lemma~\ref{lem:pi2small}, we have 
$\sum_{t \in\hat \calB}\left(\hat b_t \Pi_2(t) + \hat b_t \xi_{t, q_t}\right) = O(\sqrt{nd})$, which implies 
\begin{align*}
    \frac{\sum_{t \in \hat \calB}\left(\hat b_t\Pi_2(t) + \hat b_t \xi_{t, q_t}\right)}{\sum_{t \in \hat \calB}\hat b_t \hat \Pi_1(t)} = O\left(\log d\sqrt{\frac{\ell}{n}}\right).
\end{align*}

Therefore, 
\begin{align*}
    \hat \mu_1 = (1+O(\tau))\mu_1 + O\left(\log d \sqrt{\frac{\ell}{n}}\right), 
\end{align*}
where $\tau = \log^6 n \sqrt{\gamma}$.

\myparab{Part 2.3. Analysis when observations are from $g(\cdot)$.} We assume that the process is generated by $g(\cdot)$ instead of $\tilde g(\cdot)$. We aim to understand how the estimator changes. To distinguish the observations produced from two ``worlds'', we let 

\begin{align*}
    \my^{(1)}_{t, i}  & = \sum_{j \leq d}K_{i,j}\tilde g(\mx_{t,j}) + \xi_{t, i}, \\
    \my^{(2)}_{t, i}  & = \sum_{j \leq d}K_{i,j}g(\mx_{t,j}) + \xi_{t, i}. 
\end{align*}

Let their corresponding estimators be $\mu^{(1)}_1$ and $\mu^{(2)}_1$. We next bound the difference between these two estimators. Our crucial observation is that each $\tilde g(\mx_{t, i}) - g(\mx_{t, i})$ are bounded zero mean independent random variables. Seeing that $|\hat \mu_1^{(1)} - \hat \mu^{(2)}_1| = O(\sqrt{nd})$. Therefore, we still have
\begin{align*}
    \hat \mu_1^{(2)} = \left(1 + O(\sqrt{\gamma}\log^6 n)\right)\mu_1 + O\left(\log d\sqrt{{\frac{\ell}{n}}}\right). 
\end{align*}
This proves the second part of the theorem. 

\myparab{Remark.} We use only 1 observation for each day because our analysis relies on different $\Pi^{(i)}_1(t)$ and $\Pi^{(i)}_2(t)$ are being independent. 
The FlipSign algorithm does not need more samples because $K$ is near low-rank (Theorem~\ref{lem:decay}). 
\section{Estimating $g(\cdot)$ with boosting}

As shown in Alg.~\ref{fig:linearboost}, \Linpvel's weak learner first performs a variable selection (i.e., selects the 3 features that correlate the most with the residual returns), and then fits a linear model with both linear and quadratic interaction terms over the selected variables. The final model is a \textbf{linear} one with features and their interactions terms.    

\begin{minipage}[t]{0.9\linewidth}
    \centering
    \vspace{-0.58cm}
    \begin{algorithm}[H]\scriptsize
    \caption{{\Linpvel}}\label{fig:linearboost}
	\hspace*{\algorithmicindent} \textbf{Input} $\mX$, $\mY$, $\hat K$, $\eta$, $b$\\
    \hspace*{\algorithmicindent} \textbf{Output}  $\{g_m(\cdot)\}_{m \leq b}$
    \begin{algorithmic}[1]
        \Procedure{$\proc{Boosting-Algorithm}$}{$\mY$, $\mX$, $\hat K$, $\eta$, $b$}
        \State $\mY_{\Res} \gets \mY$ 
        \Comment $\eta$ is the learning rate
        \ForAll{$m \gets 1$ \To $b$}
        \State $g_m \gets \proc{Linear-Fit}(\mY_{\Res}, \mX, \hat{K})$
        \State $(\hat \mY)_{m} \gets \hat K g_{m}(\mX)$.
        \State $\mY_{\Res} \gets \mY_{\Res} - \eta (\hat \mY)_{m}$ 
        \EndFor
        \State \Return $\{g_m(\cdot)\}_{m \leq b}$
        \EndProcedure
        \Procedure{$\proc{Linear-Fit}$}{$\mY$, $\mX$, $\hat K$}
        \Comment $\mx_{t, i} \in \reals^{k}$ and
         $\mathbf{F}^{(t)} \in \reals^{k \times d}$
        \ForAll{$i \gets 1$ \To $d$}
        \State $\mathbf{F}^{(t)}_{:,i} = \sum_{j \in [d]}\hat K_{i, j}\mx_{t, j}$
        \EndFor
        \ForAll{$j\gets 1$ \To $k$}
        \State $r_j = \sum_{t \leq n}\mathrm{corr}(\mathbf{F}^{(t)}_{j, :}, \my_{t})$.
        \EndFor
        \State Let $j_1, j_2, j_3$ be the indices with the largest $r_j$.
        \State $g(\cdot) = \arg \min_{\beta_1, \dots, \beta_6}\sum_{\substack{t \leq n \\ i \in [d]}}(\my_{t, i} - \sum_{j \in [d]}$
        \Statex $\qquad \qquad \hat K_{i,j}\left(\beta_1 (\mx_{t, j})_{j_1} + \dots + \beta_6(\mx_{t, j})_{j_2}\cdot (\mx_{t, j})_{j_3}))\right)^2$.
        \Comment{Fits a linear model with linear and quadratic interaction terms.}
        \State \Return $g(\cdot)$ 
        \EndProcedure
    \end{algorithmic}
    \end{algorithm}
\end{minipage}

\section{Consolidation/Ensemble model}\label{sec:consolidateapp}
We next describe how we consolidate forecasts generated by multiple models. 
We do not intend to design a new consolidation algorithm. Instead, we use a ``folklore'' algorithm that weighs each model forecast by its recent historical performance.  Specifically, we let $C = \{\hat \my_1, \dots, \hat \my_m\}$ be the set of models to be consolidated. Our consolidated forecast is a linear combination of all forecasts $\hat \my = \sum_{j = 1}^m w_j \hat \my_j$, where $w_j$ is simply the $t$-statistics of the $j$-th model computed through the Newey-West estimation algorithm from the in-sample data. For example, when $m = 2$ and the $t$-statistics for $\hat \my_1$ and $\hat \my_2$ are 3 and 5 respectively, we set the consolidated forecast be proportional to $3 \times \hat \my_1 + 5 \times \hat \my_2$. The consolidated forecast needs to be properly re-scaled (e.g., set the daily standard deviation to be constant). In the forecasting models, correlation with the ground-truth is more important than MSE.  Therefore, the scale of a forecast is less important than its direction. This $t$-statistics based consolidation algorithm is used in the following two situations. 

\myparab{Allowing nparam-gEST to use all factors.} Our theoretically sound algorithm in Sec.~\ref{sec:ourAlgos} allows us to use only a small number of features. Now we may use a two-step procedure to let this algorithm simultaneously use hundreds of technical factors constructed in-house.  \emph{Step 1.} For each improvable factor, we build a model that uses only this factor as the feature. \emph{Step 2.} After we obtain multiple models (the number of models is the same as the number of improvable factors), we use the above consolidation algorithm to produce the final forecast. 

\myparab{Consolidating multiple models.} We also use the consolidation trick to aggregate the forecasts of all our models ({\Linpvel}, nparam-gEST, MLP). The result in Table~\ref{main_table} (last line) shows that the consolidated signal is stronger than any individual signal. Even if a model may not have the best out-of-sample performance, it may still be useful for constructing consolidated signals. 

\section{Additional proofs and calculations}
In this section we give some additional proofs and calculations for App.~\ref{asec:estk} and App.~\ref{asec:estg}.
\subsection{Proof of Proposition~\ref{prop:yykk}}\label{sec:yykk}

We can see that 

\begin{align}
    \frac 1 n \mY^{\transpose} \mY & = \frac 1 n\left(K^{\transpose}\mS^{\transpose}\mS K + K^{\transpose}\mS^{\transpose} E + E^{\transpose} \mS K + E^{\transpose}E\right) \\
    & = K^{\transpose}K + \cale_1 + \cale_2 + \cale_3 + \cale_4, 
\label{th:decompYtY}    
\end{align}
where $\cale_1  = K^{\transpose}\left(\frac{\mS^{\transpose} \mS}{n} - I\right)K$, $
    \cale_2  = \frac{K^{\transpose}\mS^{\transpose}E}{n} $, 
    $\cale_3  = \frac{E^{\transpose}\mS K}{n}$, and $\cale_4  = \frac{E^{\transpose}E}{n}$

We next show that each $\cale_i$ ($i \leq 4$) is small. 

\myparab{Bounding $\cale_1$.} 
We need the following lemma. 

\begin{lemma} Let 
$\mS \in \reals^{n \times d}$ be such that each row $\mS_{i, :}$ is an i.i.d. random vector $\|S_{i, :}\|_{\infty} \leq 1$ and 
$\E[\mS^{\transpose}_{i, :}\mS_{i, :}] = I$. We have 
\begin{equation}
    \Pr\left[\left\|\frac{\mS^{\transpose}\mS}{n} - I_{d \times d}\right\|_2 \geq \epsilon\right] \leq 2 n^2 \exp\left(-\frac{n \epsilon^2}{\log^4n}\right),
\end{equation}
where $\epsilon$ is a tunable parameter. 
\end{lemma}

Here, we shall set $\epsilon = \frac{\log^3 n}{\sqrt n}$. This implies that with high probability $\left\|\frac{\mS^{\transpose} \mS}{n} - I_{d \times d}\right\|_2 = O\left(\frac{\log^3 n}{\sqrt n}\right)$. On the other hand, we can see that $\|K\|^2_F = \Theta(d^2)$ and $\|K\|_F = \Theta(d)$. This implies 

\begin{equation}
    \|\cale_1\|_F = \left\|K^{\transpose}\left(\frac{\mS^{\transpose} \mS}{n} - I_{d \times d}\right)K\right\|_F \leq \left\|\frac{\mS^{\transpose}\mS}{n} - I_{d \times d}\right\|_2 \|K\|^2_F = \Theta\left(\frac{d^2 \log n}{\sqrt n}\right)
\end{equation}

\myparab{Bounding $\cale_2$ and $\cale_3$.} 
Recall that $\cale_2 = \frac{K^{\transpose}\mS^{\transpose}E}{n}$ and $\cale_3 = \frac{E^{\transpose}\mS K}{n}$. We have the following lemma.

\begin{lemma}\label{lem:se}Let $\mS \in \reals^{n \times d}$ be such that each row $\mS_{i, :}$ is an i.i.d. random vector with $\|\mS_{i, :}\|_{\infty} \leq 1$ and $\E[\mS^{\transpose}_{i, :}\mS_{i, :}] = I$.  Let $E \in \reals^{n \times d}$ be such that $E_{i,j}$ are i.i.d. Gaussian with standard deviation $\sigma_{\xi}$. We have with overwhelming probability
\begin{equation}
    \|\mS^{\transpose} E\|^2_F \leq c_0 \sigma_{\xi}d^2 n
\end{equation}
for some constant $c_0$. 
\end{lemma}
\begin{proof}[Proof of Lemma~\ref{lem:se}] First, note that 
\begin{align*}
    \E[\|\mS^{\transpose}E_{:, i}\|^2_F \mid \mS] = \sigma^2_{\xi}\|S\|^2_F. 
\end{align*}
Therefore, we have $\E[\|\mS^{\transpose}E_{: ,i} \|^2_F] = \sigma^2_{\xi}d n$. This also implies that 
\begin{align*}
    \E[\|\mS^{\transpose}E\|^2_F] = \sum_{i \leq d}\E[\|\mS^{\transpose}E_{:,i}\|^2] = \sigma^2_{\xi}d^2n. 
\end{align*}
By a standard Chernoff bound, we have whp 
\begin{equation}
    \|\mS^{\transpose}E\|^2_F \leq c_0 \sigma^2_{\xi}d^2 n 
\end{equation}
for some constant $c_0$, i.e., whp $\|\mS^{\transpose}E\|_F = O(\sigma_{\xi}d \sqrt{n})$. 
\end{proof}

We next use Lemma~\ref{lem:se} to bound $\cale_2$ and $\cale_3$:
\begin{equation}
    \|\cale_2 \|_F = \|\cale_3\|_F = \frac 1 n \|K^{\transpose}\mS^{\transpose}E\|_F = \frac 1 n \|K\|_2 \|\mS E\|_F. 
\end{equation}

Now we have $\|K\|_2 = O(d)$ and $\|\mS E\|_F = O(d \sqrt n)$ whp. Therefore, with high probability 
\begin{equation*}
    \|\cale_2 \|_F = \|\cale_3 \|_F = O\left(\frac{d^2}{\sqrt n}\right).
\end{equation*}

\myparab{Bounding $\cale_4$.} With the assumption that $d = O(n)$, we have 
\begin{equation}
    \left\|\frac{E^{\transpose} E}{n}\right\|^2_F \leq \frac{\Rank(E^{\transpose}E)\|E\|^2_2}{n} = O\left(\frac{\sigma^2_{\xi}dn}{n}\right) = O(\sigma^2_{\xi}d).  
\end{equation}

Above, we used a finite sample version of semi-circle law (i.e., $\|E\|^2_2 = O(n)$ whp~\cite{rudelson2010non}).

Summing up above and using that  $n < d^2$ and $\sigma_{\xi} = O(\sqrt{d})$, we have $\left\|\frac 1 n\mY^{\transpose}\mY - K^{\transpose}K\right\|_F = O\left(\frac{d^2 \log^3 n}{\sqrt n}\right).$

\subsection{Anti-concentrations}
\begin{theorem}(Littlewood-Offord-Erdos; e.g.,~\cite{krishnapur2016anti})
Let $L_1, \dots, L_d \geq 1$. Let $\xi_1, \dots \xi_n$ be independent Bernoulli $\pm 1$ unbiased random variables such that $\Pr[\xi_i = 1] = \frac 1 2$. Let $S = \sum_{i \leq n}\xi_i L_i$. For any open interval $I$ of length 2, we have 
\begin{align}
    \Pr[S \in I] = O(n^{- \frac 1 2}). 
\end{align}
\label{thm:anti}
\end{theorem}

\begin{lemma}\label{lem:anti}
Let $\ell \leq d/\log^2d$. 
Let $L_1, L_2, \dots, L_d$ be positive numbers such that $L_i = \Omega(1)$. Define a random variable 
\begin{align*}
    Z_i = \left\{
    \begin{array}{ll}
    L_i & \mbox{ with probability } \frac 1 {\ell}\\
    - \frac{L_i}{\ell-1} & \mbox{ with probability } 1- \frac{1}{\ell}.
    \end{array}
    \right.
\end{align*}
There exist constants $c_0$ and $c_1$ such that 
\begin{align*}
    \Pr\left[\sum_{i \leq d}Z_i \geq \frac{c_0}{\log d}\sqrt{\frac{d}{\ell}}\right] \geq c_1
\end{align*}
\end{lemma}

\begin{proof} We shall use Theorem~\ref{thm:anti} to prove Lemma~\ref{lem:anti}. Theorem~\ref{thm:anti} requires that 
random variables $\xi_i$ (or $Z_i$ in our setting) to be symmetric, 
which is violated in our setting. Our goal is to reduce our problem to the original setting. 

We now show that this can be done through ``debiasing'' the walk. 
We first define $\{B_i\}_{i \in [d]}$ such that $B_i$ is a random  binary indicator variable with $\Pr[B_i = 1] = \frac{\ell-2}{\ell}$ and $\Pr[B_i = 0] = \frac{2}{\ell}$. 

We may generate $Z_i$ by using $B_i$, i.e., when $B_i = 1$, we set $Z_i = -\frac{L_i}{\ell - 1}$, and when $B_i = 0$, we set $Z_i = L_i$ with half of the probability and $Z_i = -\frac{L_i}{\ell- 1}$ with the other half of the probability. Note that when $B_i = 0$, the probability that $Z_i$ takes one of the possible values in $\frac 1 2$ (thus is uniform).

Next, let $\calB = \{B_i: B_i = 1\}$ and $\bar \calB = \{ B_i: B_i = 0\}$. Let also that $T = |\bar \calB|$. One can see that $\E[T] = \frac{2d}{\ell}$. In addition, because $d = \omega(\ell \log \ell)$, with overwhelming probability that $T \geq \frac{d}{\ell}$. 

We now can see that 
\begin{align*}
    \E\left[\sum_{i \in \calB}Z_i\right] = - \left(\sum_{i \leq d}L_i\right)\frac{\ell-2}{\ell(\ell-1)}
\end{align*}

In addition, $\E[Z_i \mid i \in \bar \calB] = L_i\left(1- \frac{1}{\ell - 1}\right)\frac 1 2$ for any $i \in \bar \calB$.
Next, we define a random variable to ``debias'' $Z_i$, conditioned on $i \notin \calB$, i.e., for any $i \in\bar \calB$

\begin{equation}
    \tilde Z_i = \left\{\begin{array}{ll}
L_{i} - L_i\left(1 - \frac{1}{\ell - 1}\right)\frac 1 2 & \quad \mbox{with probability $\frac 1 2$}    \\
-\frac{L_i}{\ell-1} - L_i\left(1 - \frac{1}{\ell - 1}\right)\frac 1 2 & \quad \mbox{with probability $\frac 1 2$.} 
\end{array} \right. 
\end{equation}
Note that $\E[\tilde Z_i] = 0$ and $L_{i} - L_i\left(1 - \frac{1}{\ell - 1}\right)\frac 1 2 = 
\frac{L_i}{\ell-1} + L_i\left(1 - \frac{1}{\ell - 1}\right)\frac 1 2 $. Next, we have 
\begin{align*}
    \sum_{i \leq d}Z_i &= \sum_{i \in \calB}\left[- \frac{L_i}{\ell-1}\right] + \sum_{i \in \bar \calB}\left(\tilde Z_i + L_i\left(1 - \frac{1}{\ell - 1}\right)\frac 1 2\right) \\
    & =  \underbrace{\left(\sum_{i \in \calB} \left(-\frac{L_i}{(\ell-1)}\right) + \sum_{i \notin \calB}  L_i\left(1 - \frac{1}{\ell - 1}\right)\frac 1 2\right)}_{\Psi_1} + \underbrace{\left(\sum_{i \notin \calB}\tilde Z_i\right)}_{\Psi_2}. 
\end{align*}

One can see that \emph{(i)} the sign of $\Psi_2$ is independent of the sign of $\Psi_1$, and \emph{(ii)} one of $\Pr[\Psi_1 \geq 0] \geq \frac 1 2$
 and $\Pr[\Psi_1 \leq 0] \geq \frac 1 2$ must hold. Wlog, assume that $\Pr[\Psi_1 \geq 0]\geq \frac 1 2$. By Theorem~\ref{thm:anti}, we have 
 $\Pr\left[\Psi_2 \geq \frac{c_1}{\log d}\sqrt{T}\right] = \Omega(1)$.
 
Finally, we have 

\begin{align*}
    \Pr\left[\Psi_1 + \Psi_2 \geq \frac{c_1}{\log d}\sqrt{\frac{d}{\ell}}\right] & \geq \Pr[\Psi_1 \geq 0]\Pr\left[\Psi_2 \geq \frac{c_2}{\log d}\sqrt T \mid \Psi_1 \geq 0\right]\\
    & \geq \Pr[\Psi_ \geq 0]\Pr\left[\Psi_2 \geq \frac{c_1}{\log d}\sqrt{\frac{d}{\ell}} \mid \Psi_1 \geq 0\right] \\
    & = \Omega(1). 
\end{align*}

The second inequality uses $T \geq \frac{d}{\ell}$ whp. 
\end{proof}

\section{Experiments}\label{sec:app_exp}
We evaluate our algorithms on an emerging market dataset and a social network dataset.  We describe the
dataset collection and setup of experiments, the evaluation metrics, additional explanation of baselines, and analysis for our performance for both datasets. 


\subsection{Equity returns}\label{setting}
We use daily prices and volumes to generate the features and focus on predicting the next 5-day returns. 

\myparab{Datasets collection.}\label{dataset}
The specific description of the used dataset is as follows:

\myparab{(1) Chinese stock data:}
Our data set consists of daily prices and trading volumes of approximately 3,600 stocks between 2009 and 2018. We use open prices to compute the returns and we aim to predict the next 5-day returns, in which the last three years are out-of-sample. 
We examine two universes. 
(i) \emph{Universe 800} is equivalent to the S$\&$P 500 and consists of 800 stocks, and (ii) \emph{Full universe} consists of all stocks except for illiquid ones.
The average ``size'' (in either capital or trading volume) in \emph{Universe 800} is larger than the average ``size'' of the \emph{Full universe}.

\myparab{(2) Technical factors:}
We manually build 337 technical factors based on previous studies~\cite{gu2020empirical,colby1988encyclopedia,kakushadze2016101,amihud2002illiquidity,posner2014economic}. All these factors are derived from price and dollar volume.

 \myparab{(3) Barra factor dataset:} 
 We use a third-party risk model known as the Barra factor model~\cite{orr2012supplementary}. The model uses 10 real-valued factors and 1 categorical variable to characterize a stock.
 The real-valued factors known as ``style factors'' include beta, momentum, size, earnings yield, residual volatility, growth, book-to-price, leverage, liquidity, and non-linear size. The categorical variable represents the industrial sector the stock is in. We do not use the categorical variable in our experiments. Table~\ref{tab:barrafactors} defines the style factors.
 
 
 
 \begin{table}[!ht]
    \centering
    \begin{adjustbox}{width=1\columnwidth,center}
    \begin{tabular}{l|l|l|l|l|l}
    \hline
    Barra factors name & \textbf{Beta}   & \textbf{Momentum}      & \textbf{Size}     & \textbf{Earnings Yield} & \textbf{Residual Volatility} \\ \hline
    Description &
      \begin{tabular}[c]{@{}l@{}} Measure of \\ volatility.\end{tabular} &
      \begin{tabular}[c]{@{}l@{}} Rate of acceleration of \\ a security's price or volume.\end{tabular} &
      \begin{tabular}[c]{@{}l@{}} Total equity \\ value in market.\end{tabular} &
      \begin{tabular}[c]{@{}l@{}}The percentage of how much \\ a company earned per share.\end{tabular} &
      \begin{tabular}[c]{@{}l@{}}The volatility of daily \\ excess returns.\end{tabular} \\ \hline
    Barra factors name & \textbf{Growth} & \textbf{Book-to-Price} & \textbf{Leverage} & \textbf{Liquidity}      & \textbf{Non-linear Size}     \\ \hline
    Description &
      \begin{tabular}[c]{@{}l@{}} Measure of \\ the growth rate.\end{tabular} &
      \begin{tabular}[c]{@{}l@{}} firm's book value to its \\ market capitalization.\end{tabular} &
      \begin{tabular}[c]{@{}l@{}} Measure of a firm's\\ leverage rate.\end{tabular} &
      Measure of a firm's liquidity. &
      \begin{tabular}[c]{@{}l@{}}Non-linear transformation\\ of size factor.\end{tabular} \\ \hline
    \end{tabular}
    \end{adjustbox}
    \vspace{2mm}
    \caption{Barra style factors from~\protect\cite{orr2012supplementary}.}
    \label{tab:barrafactors}
\end{table}

\myparab{(4) News dataset:} We crawled financial news between 2012 and 2018 from a major Chinese news website Sina. We collected a total number of 2.6 million news articles. 
Each article can refer to one or multiple stocks. On average, a piece of news refers to 2.94 stocks. We remark that our way to use news data sets deviates from standard news-based models for predicting equity returns~\cite{ding2015deep,hu2018listening}. Most news-based models aim to extract sentiments and events that could directly impact one or more related stocks' prices. Rather than building links between events and the stock fluctuation, we use news dataset to identify similarities between stocks. i.e., when two stocks are mentioned often, they are more likely to be similar. This is orthogonal to how the news itself impacts the movement of stock prices.  



\myparab{Model and training.}\label{train_test}
We use three years of data for training, 10 months of data for validation and one year of data for testing. 
We re-train the model every testing year. For example, the training set starts from Jan. 1, 2012, to Dec. 31, 2014. The corresponding validation period is from Jan. 15, 2015, to Dec. 16, 2015. We use the validation set to select the hyperparameters and build the model. Then we use the trained model to forecast returns of equity in the same universe from Jan. 1, 2016, to Dec. 31, 2016, where we set 10 trading days as the ``gap''. Then we re-train the model by using data in the second training period (Jan. 1, 2013, to Dec. 17, 2015).  We set a ``gap'' between the training and validation periods, and the validation periods testing dataset to avoid looking-ahead issues.

\subsection{Additional explanation about evaluation matrices and baselines} \label{robustevaluation}



\myparab{Computing $t$-statistics.}  Recall that $\my_t \in \reals^d$ is a vector of responses and $\hat \my_t \in \reals^d$ is the forecast of a model to be evaluated. We examine whether the signals are correlated with the responses, i.e., for each $t$ we run the regression model $\my_t = \beta_t \hat \my_t + \epsilon$ and test whether we can reject the null hypothesis that the series $\beta_t = 0$ for all $t$. Note that the noises in the regression model are serially correlated so we use Newey-West~\cite{newey1986simple} estimator to adjust serial correlation issues. 
Consider, for example, a coin-tossing game, in which we make one dollar if our prediction of a coin toss is correct or lose one dollar otherwise. When our forecast has 51\% accuracy, we are guaranteed to generate positive returns in the long run by standard concentration results. Testing whether our forecast has better than 51\% accuracy needs many trials because, e.g., when there are only 100 tosses, 
there is a $\approx 40\%$ probability that a random forecast has a $\geq 51\%$ accuracy rate.

\myparab{An example of compare correlation vs MSE.}
Consider a case where the true returns of Google and Facebook are +2\% and +4\%, respectively. Let forecast A be -1\% (Google) and -1\% (Facebook), and let forecast B be +20\% (Google) and +40\% (Facebook). While forecast A has a smaller MSE, forecast B is more accurate and more profitable (e.g., the directions of the returns are predicted correctly).

\myparab{Sharpe Ratio. }
The popular Sharpe Ratio measures the performance of an investment by adjusting for its risk.
\begin{equation}
    \text{Sharpe Ratio} = \frac{R_p -R_f}{\sigma_p}, 
\end{equation}
where $R_p$ is the return of the portfolio, $R_f$ is the risk-free rate, and $\sigma_p$ is the standard deviation of the portfolio's excess return. 

\myparab{PnL.}
Profit \& Loss (PnL) is a standard performance measure used in trading and captures the total profit or loss of a portfolio over a specified period. The PnL of all forecasts made on day $t$ is given by 

\begin{equation}
\text{PnL} =\frac{1}{d}\sum_i^d  sign(\hat{\my}_{t,i}) * \my_{t,i} ,\quad t = 1,\ldots,n, 
\end{equation}


\myparab{Additional explanation for recent CAMs}\label{baselines}

\myparab{$\sbullet[.75]$ SFM~\cite{zhang2017stock}.} \textsc{SFM} decomposes the hidden states of an LSTM~\cite{rather2015recurrent} network into multiple frequencies by using Discrete Fourier Transform (DFT) so the model can capture signals at different horizons. 

\myparab{$\sbullet[.75]$ HAN~\cite{hu2018listening}.} This work introduces a so-called hybrid attention technique that translates news into signals. 

\myparab{$\sbullet[.75]$ AlphaStock~\cite{wang2019alphastock}.}
This work is proposed by~\cite{wang2019alphastock}. 
AlphaStock integrates deep attention networks reinforcement learning with the optimization of the Sharpe Ratio. For each stock, AlphaStock uses LSTM~\cite{sak2014long} with attention on hidden states to extract the stock representation. Then AlphaStock uses CAAN, which is a self-attention layer, to capture the interrelations among stocks. Specifically, CAAN 
takes the stock representations as inputs to generate the stock's winning score. We implement LSTM with basic CAAN and change the forecast into return instead of winning scores. 

\myparab{$\sbullet[.75]$ ARRR} ARRR~\cite{wu2019adaptive} is a new regularization technique designed to address the overfitting issue in vector regression under the high-dimensional setting. Specifically, ARRR involves two SVD, the first SVD is for estimating the precision matrix of the features, and the second SVD is for solving the matrix denoising problem.

\subsection{Experiment evaluation}~\label{performance}
\vspace{-4mm}

\myparab{Detailed results for each testing year}
Tables~\ref{by_year_800} and~\ref{by_year_full} list the results for each testing year in \emph{Universe 800} and \emph{Full universe}. The bold fonts denote the best performance in each group. The results are consistent with the Table~\ref{main_table}. Note that we also report weighted correlation and weighted t-statistic. The weights are determined by the historical dollar volume of the asset. These statistics are useful because the positions taken by the optimizer are sensitive to historical dollar volumes.

\begin{table}[ht!]
\centering
\begin{adjustbox}{width=0.9\columnwidth,center}
\begin{tabular}{l|l|llll|llll|llll} 
\hline
\multicolumn{1}{l}{} &       & \multicolumn{4}{l|}{2016}                                                      & \multicolumn{4}{l|}{2017   }                                                      & \multicolumn{4}{l}{2018}                                                       \\ 
\hline
                     &   Our CAMs    & core              & w\_corr           & t-stat             & w\_t-stat         & corr               & w\_corr            & t-stat             & w\_t-stat          & corr              & w\_corr           & t-stat            & w\_t-stat          \\ 
\hline
{\Linpvel}                  & Opt.  & 0.1084            & \textbf{0.1149 }  & 9.9081             & \textbf{7.9814 }  & \textbf{0.0388 }   & \textbf{0.0624 }   & \textbf{3.0441 }   & \textbf{3.8233 }   & \textbf{0.0820 }  & \textbf{0.1037 }  & \textbf{7.4293 }  & \textbf{7.2038 }   \\
                     & DD    & 0.1064            & 0.1109            & 9.7619             & 7.4212            & 0.0284             & 0.0541             & 2.1949             & 3.1475             & 0.0729            & 0.0972            & 6.9855            & 6.8733             \\
                     
\hline
nparam-gEST           & Opt.  & 0.0800            & 0\textbf{.0595 }  & \textbf{6.1201 }   & \textbf{2.9453 }  & -0.0067            & \textbf{-0.0095 }  & -0.4554            & \textbf{-0.4800 }  & 0.0604            & 0.0461            & 4.\textbf{2237 }  & 2.2608             \\
                     & DD    & 0.\textbf{0805 }  & 0.0577            & 6.0357             & 2.7922            & \textbf{-0.0051 }  & -0.0102            & \textbf{-0.3472 }  & -0.5496            & 0.0582            & \textbf{0.0446 }  & 4.1457            & \textbf{2.3772 }   \\
                    
\hline
MLP                  & Opt.  & \textbf{0.0958 }  & 0.0917            & \textbf{7.4846 }   & 5.0039            & 0\textbf{.0050 }   & 0.0182             & 0.3610             & \textbf{0.9769 }   & \textbf{0.0641 }  & 0.0602            & 5.7370            & 3.4793             \\
                     & DD    & 0.0940            & \textbf{0.0919 }  & 7.3239             & \textbf{5.0924 }  & 0.0047             & 0.0165             & 0.3308             & 0.8749             & 0.0634            & \textbf{0.0604 }  & 6.0943            & 3.5154             \\
                     
\hline
LSTM                 & Opt.  & \textbf{0.0662 }  & \textbf{0.0762 }  & 5.7290             & \textbf{4.5106 }  & \textbf{-0.0216 }  & -0.0144            & -1.5466            & -0.7624            & \textbf{0.0413 }  & \textbf{0.0423 }  & \textbf{3.7099 }  & \textbf{2.4316 }   \\
                     & DD    & 0.0606            & 0.0682            & 4.8167             & 4.0641            & -0.0025            & \textbf{0.0017 }   & \textbf{-0.1520 }  & \textbf{0.0803 }   & 0.0110            & 0.0356            & 0.9228            & 1.8596             \\
                    
\hline
Linear               & Opt.  & \textbf{0.0726 }  & \textbf{0.0708 }  & \textbf{5.1011 }   & \textbf{3.5324 }  & \textbf{0.0054 }   & 0.0166             & 0.3429             & 0.8720             & \textbf{0.0567 }  & \textbf{0.0679 }  & \textbf{4.1086 }  & \textbf{4.1605 }   \\

\hline
\multicolumn{2}{l|}{UM: poor man \Linpvel}  & \textbf{0.1093 }  & 0.1128            & \textbf{10.1352 }  & 7.5786            & 0.0242             & 0.0499             & 1.7650             & 2.9580             & 0.0688            & 0.0970            & 6.3840            & 6.6570             \\ 

\multicolumn{2}{l|}{UM: poor man nparam-gEST}   & 0.0788            & 0.0579            & 6.0820             & 2.8561            & -0.0061            & -0.0096            & -0.4083            & -0.4862            & 0.0569            & 0.0442            & 3.7779            & 2.1036             \\ 

\multicolumn{2}{l|}{UM: MLP}   & 0.0861            & 0.0812            & 5.8771             & 4.2409            & 0.0052             & 0.0132             & \textbf{0.3800 }   & 0.6764             & 0.0609            & 0.0571            & \textbf{6.1635 }  & \textbf{3.7053 }   \\ 
\multicolumn{2}{l|}{UM: LSTM}   & 0.0619            & 0.0632            & \textbf{6.5873 }   & 4.1299            & -0.0253            & -0.0215            & -1.7873            & -1.1504            & 0.0169            & 0.0183            & 1.4487            & 1.0374             \\ 
\multicolumn{2}{l|}{UM: Lasso}            & -0.0046           & 0.0088            & -0.3889            & 0.5531            & 0.0282             & \textbf{0.0333 }   & \textbf{2.1633 }   & \textbf{2.0936 }   & 0.0083            & 0.0153            & 1.2997            & 1.6726             \\

\multicolumn{2}{l|}{UM: Ridge}        & 0.0290            & 0.0406            & 3.4617             & 2.8301            & -0.0064            & -0.0161            & -1.3527            & -1.3455            & 0.0091            & 0.0066            & 1.5421            & 0.5618             \\ 

\multicolumn{2}{l|}{UM: GBRT}    & 0.0655            & 0.0601            & 9.9051             & 5.4083            & 0.0419             & 0.0565             & 5.6987             & 5.0517             & 0.0476            & 0.0606            & 7.1179            & 6.4332             \\ 
\multicolumn{2}{l|}{UM: SFM}     & 0.0114            & 0.0102            & 0.9237             & 0.6828            & 0.0097             & 0.0081             & 0.6644             & 0.4479             & 0.0078            & -0.0133           & 0.6194            & -0.8263            \\ 
\hline
\multicolumn{2}{l|}{Existing CAM: Alpha}   & 0.0132            & 0.0165            & 2.3632             & 1.8841            & 0.0135             & 0.0133             & 2.5594             & 1.6109             & -0.0062           & -0.0110           & -1.3995           & -1.3236            \\ 
\hline
\multicolumn{2}{l|}{Existing CAM: HAN}     & 0.0096            & 0.0056            & 1.0205             & 0.4777            & 0.0060             & 0.0088             & 0.4980             & 0.5455             & 0.0160            & 0.0101            & 1.9352            & 0.7273             \\ 
\hline
\multicolumn{2}{l|}{Existing CAM: VR}                & 0.0207             & 0.0038             & 1.8590              & 0.2582   & 0.0087            & 0.0192            & 0.9239             & 1.6069            & 0.0174            & 0.0248            & 1.6513            & 1.3219             \\ \hline
\multicolumn{2}{l|}{Existing CAM: ARRR}   & 0.0593       & 0.0657            & 3.8366           & 3.2866       &  -0.0083             & -0.0043            &  0.3975            &     0.578  & 0.0432 & 0.0533 & 3.3669     & 3.9343 \\ \hline
\end{tabular}
\end{adjustbox}
\caption{The by year results for \emph{Universe 800}}
\vspace{-2mm}
\label{by_year_800}
\end{table}

\begin{table}[ht!]
\centering
\begin{adjustbox}{width=0.9\columnwidth,center}
\begin{tabular}{l|l|llll|llll|llll} 
\hline
\multicolumn{1}{l}{} &      & 2016              &                   &                    &                   & 2017              &                   &                   &                   & 2018              &                   &                    &                    \\ 
\hline
Method               &   Our CAMs     & corr              & w\_corr           & t-stat             & w\_t-stat         & corr              & w\_corr           & t-stat            & w\_t-stat         & corr              & w\_corr           & t-stat             & w\_t-stat          \\ 
\hline
{\Linpvel}                  & Opt.  & 0.1328   & \textbf{0.1316 }  & 12.1131            & 8.9092            & 0.0564            & 0.0590            & 4.4505            & 3.3487            & \textbf{0.0939 }  & \textbf{0.1122 }  & 8.2186             & 7.0727             \\
                     & DD    & \textbf{0.1358}            & 0.1308            & \textbf{12.7510 }  & \textbf{9.0186 }  & \textbf{0.0584 }  & \textbf{0.0632 }  & \textbf{4.8204 }  & \textbf{3.5678 }  & 0.0859            & 0.1062            & \textbf{9.5940 }   & \textbf{8.1365 }   \\
                    
\hline
nparam-gEST           & Opt.  & \textbf{0.1045 }  & \textbf{0.0969 }  & \textbf{10.3212 }  & \textbf{6.6829 }  & 0.0159            & \textbf{0.0129 }  & 1.2465            & \textbf{0.7205 }  & \textbf{0.0650 }  & \textbf{0.0559 }  & \textbf{5.6303 }   & \textbf{3.1603 }   \\
                     & DD    & 0.1039            & 0.0941            & 8.9599             & 6.1395            & \textbf{0.0174 }  & 0.0118            & \textbf{1.3356 }  & 0.6298            & 0.0596            & 0.0463            & 4.8991             & 2.3960             \\
                     
\hline
MLP                  & Opt.  & \textbf{0.1072 }  & \textbf{0.0983 }  & 8.3802             & 6.1198            & 0.0290            & \textbf{0.0219 }  & 1.9765            & \textbf{1.0954 }  & \textbf{0.0851 }  & \textbf{0.0876 }  & \textbf{11.0013 }  & \textbf{5.9272 }   \\
                     & DD    & 0.0935            & 0.0921            & \textbf{8.4685 }   & \textbf{6.6603 }  & \textbf{0.0303 }  & 0.0212            & 2.1287            & 1.0544            & 0.0776            & 0.0788            & 8.4386             & 5.0757             \\
                     
\hline
LSTM                 & Opt.  & \textbf{0.0744 }  & \textbf{0.0750 }  & \textbf{7.0918 }   & \textbf{4.5892 }  & 0.0210            & 0.0200            & 1.3738            & 0.8513            & \textbf{0.0465 }  & \textbf{0.0523 }  & \textbf{5.6732 }   & \textbf{3.3210 }   \\
                     & DD    & 0.0476            & 0.0532            & 4.1179             & 3.9801            & \textbf{0.0327 }  & \textbf{0.0292 }  & \textbf{2.7048 }  & \textbf{1.4664 }  & 0.0441            & 0.0462            & 4.5315             & 2.5036             \\
                     
\hline

Linear               & Opt.  & \textbf{0.0995 }  & \textbf{0.0956 }  & 7.6410             & 5.7275            & 0.0123            & 0.0041            & 0.7983            & 0.1940            & \textbf{0.0527 }  & \textbf{0.0684 }  & \textbf{5.1804 }   & \textbf{4.4595 }   \\

\hline
\multicolumn{2}{l|}{UM: poor man \Linpvel}   & 0.1279            & 0.1214            & 11.4010            & 8.2082            & 0.0488            & 0.0517            & 3.6993            & 2.8182            & 0.0713            & 0.0920            & 7.1887             & 5.9714             \\ 

\multicolumn{2}{l|}{UM: poor man nparam-gEST}   & 0.1002            & 0.0957            & 8.8982             & 6.2661            & 0.0169            & 0.0112            & 1.2822            & 0.5872            & 0.0580            & 0.0457            & 4.8490             & 2.4000             \\

\multicolumn{2}{l|}{UM: MLP}   & 0.0837            & 0.0830            & 6.3470             & 5.4216            & 0.0286            & 0.0123            & \textbf{2.3251 }  & 0.6869            & 0.0697            & 0.0449            & 8.1207             & 2.5859             \\ 

\multicolumn{2}{l|}{UM: LSTM}  & 0.0684            & 0.0577            & 5.7502             & 3.6079            & -0.0007           & -0.0057           & -0.0401           & -0.2410           & 0.0379            & 0.0370            & 3.4094             & 1.8486             \\ 

\multicolumn{2}{l|}{UM: Lasso}      & 0.0589            & 0.0612            & \textbf{9.4412 }   & \textbf{8.2446 }  & -0.0032           & -0.0028           & -0.3080           & -0.1819           & 0.0313            & 0.0169            & 2.6735             & 0.8430             \\
\multicolumn{2}{l|}{UM: Ridge}         & 0.0631            & 0.0636            & 5.9308             & 4.4291            & \textbf{0.0152 }  & \textbf{0.0168 }  & \textbf{0.9798 }  & \textbf{0.8296 }  & 0.0290            & 0.0413            & 2.7536             & 2.2439             \\ 
\multicolumn{2}{l|}{UM: GBRT}    & 0.0898            & 0.0842            & 13.6203            & 9.5775            & 0.0531            & 0.0687            & 5.8328            & 7.0454            & 0.0588            & 0.0711            & 8.5605             & 7.0563             \\ 
\multicolumn{2}{l|}{UM: SFM}     & -0.0055           & -0.0058           & -0.4868            & -0.4281           & 0.003             & 0.0096            & 0.3578            & 0.708             & 0.0107            & 0.0057            & 1.2447             & 0.4851             \\ 
\hline
\multicolumn{2}{l|}{Existing CAM: Alpha}   & 0.0076            & 0.0109            & 1.2236             & 1.3496            & 0.0093            & 0.0123            & 2.3817            & 1.9694            & 0.003             & 0.008             & 0.778              & 1.4379             \\ 
\hline
\multicolumn{2}{l|}{Existing CAM: HAN}     & 0.0135            & 0.0081            & 1.7515             & 0.7924            & -0.0008           & -0.0020           & -0.0791           & -0.1253           & 0.0114            & 0.0098            & 1.5316             & 0.7547             \\ 
\hline
\multicolumn{2}{l|}{Existing CAM: VR}      & 0.0031            & 0.0033            & 0.3887             & 0.2949            & -0.0056           & -0.0245           & -0.7108           & -1.868            & 0.0148            & 0.0138            & 2.0156             & 0.8331             \\
\hline
\multicolumn{2}{l|}{Existing CAM: ARRR}      & 0.0527            & 0.0714          & 2.9072             & 3.2284            & -0.0067          & -0.0169                   & 0.6266            & -0.2307           & 0.0205           & 0.0275             & 2.0334 & 1.2328         \\ \hline
\end{tabular}
\end{adjustbox}
\caption{Yearly results for \emph{Full universe}.}
\label{by_year_full}
\end{table}

\myparab{Simulation and PnL.} 
Fig.~\ref{appfig:fret_quantiles} shows three ways to simulate investments on our signals for testing years from 2016 to 2018 for   \emph{Universe 800} and \emph{Full universe}. 
(i) Long-index portfolio: Long-only minus the market index. (ii) Long-short portfolio\footnote{Short is implementable in the Chinese market only under special circumstances, e.g., through brokers in Hong Kong under special arrangements.}: By allowing short-selling,  we can execute on negative forecasts to understand the overall forecasting quality. (iii) Weighted-Long-short portfolio: We weight an investment by the historical turnover of the asset. We conduct the trading in the daily granularity and select the stocks from the top 20\% strongest forecast signals. We can see that our signals are consistently better than other baselines in both long/long-short. The results confirm that our method generates stronger and more robust signals for trading.   
\begin{figure*}[h!]
    \centering
    \subfigure[\textit{Long-index}.]{
    \includegraphics[width=0.3\textwidth]{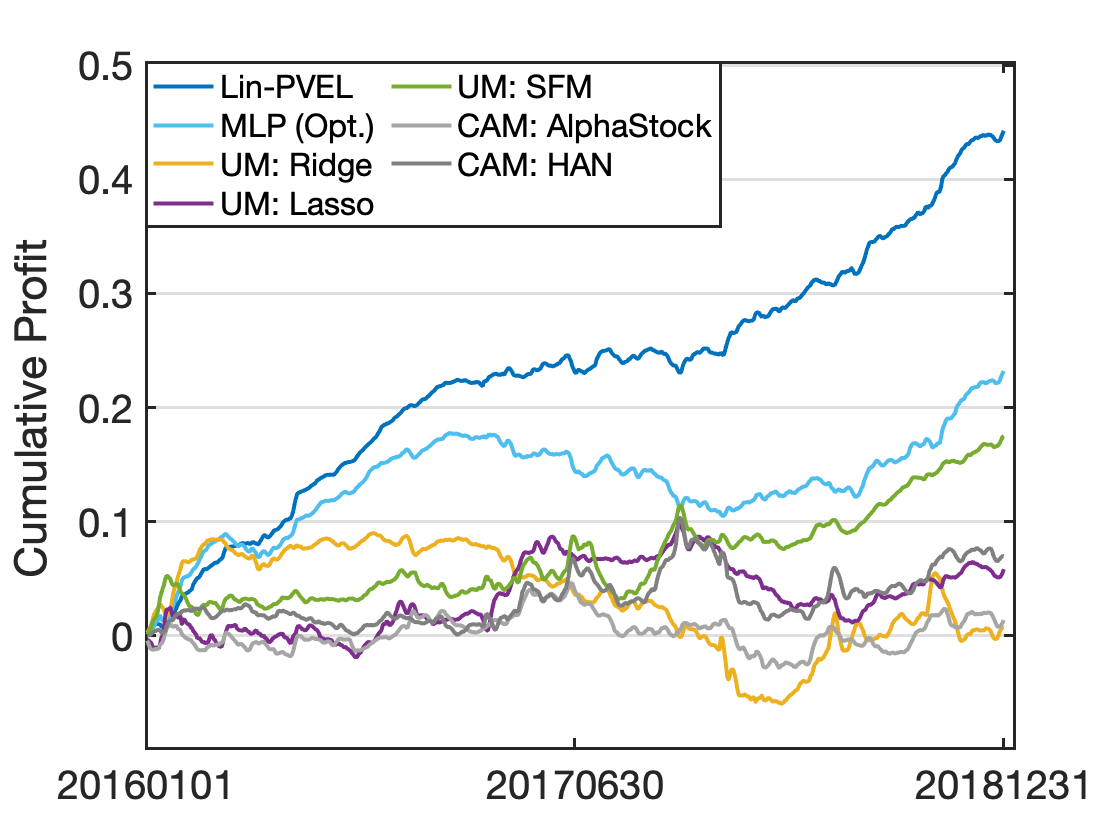}
    \label{appfig:quant_li_800}}  
    \subfigure[\textit{Long-short}.]{
    \includegraphics[width=0.3\textwidth]{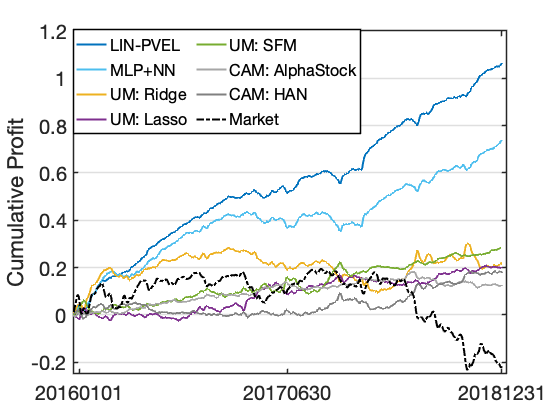}
    \label{appfig:quant_ls_800} }     
    \subfigure[\textit{Weighted-long-short}.]{
    \includegraphics[width=0.3\textwidth]{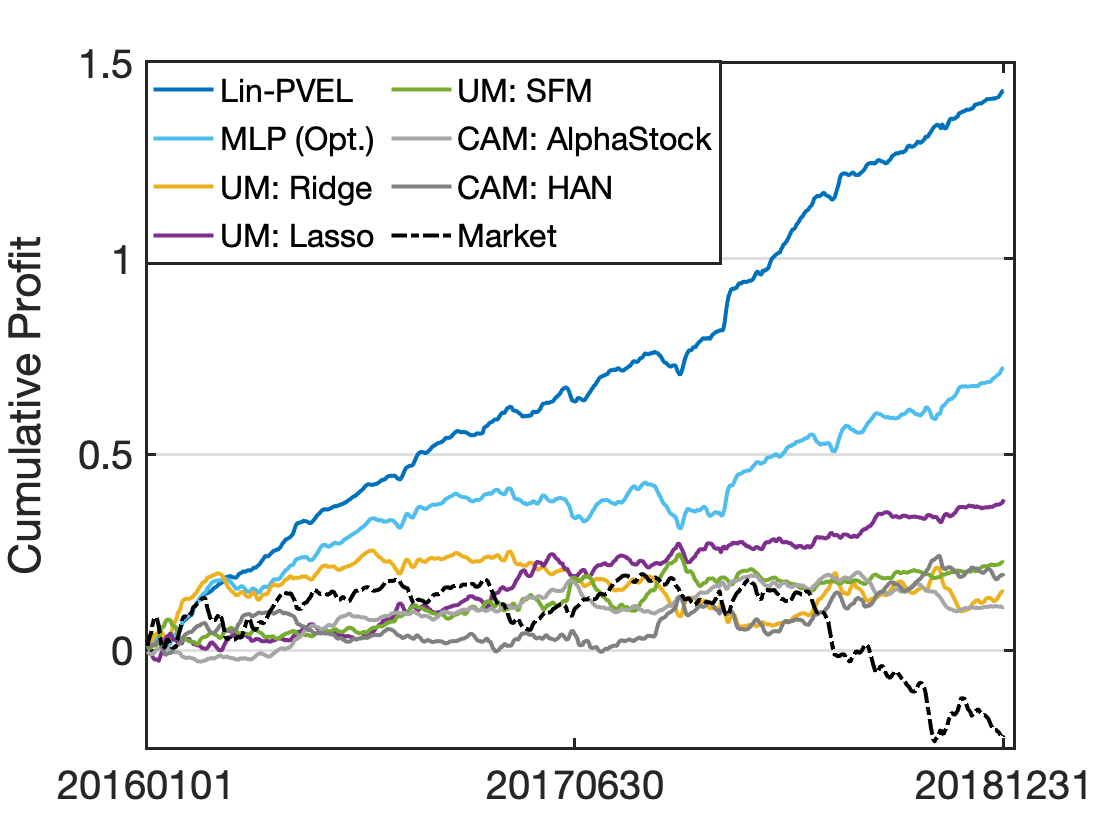}
    \label{appfig:quant_wsl_800}}
     \subfigure[\textit{Long-index}.]{
    \includegraphics[width=0.3\textwidth]{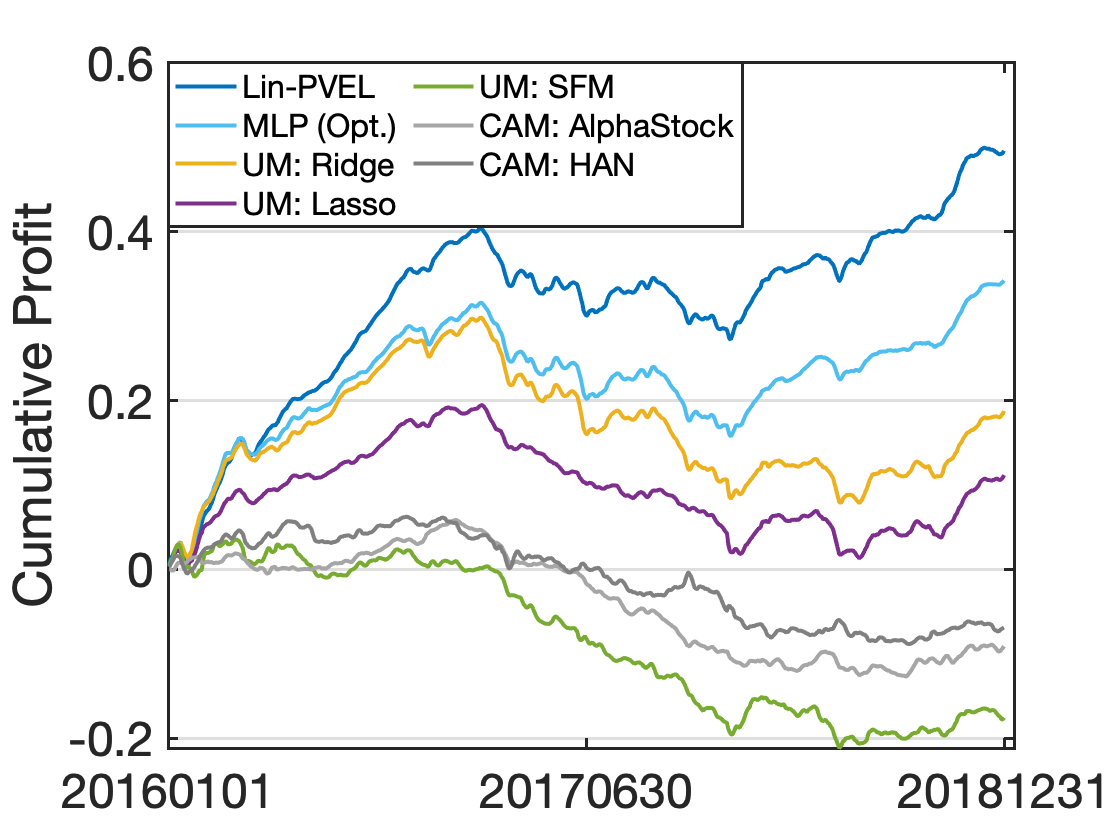}
    \label{appfig:quant_li_3000}}  
    \subfigure[\textit{Long-short}.]{
    \includegraphics[width=0.3\textwidth]{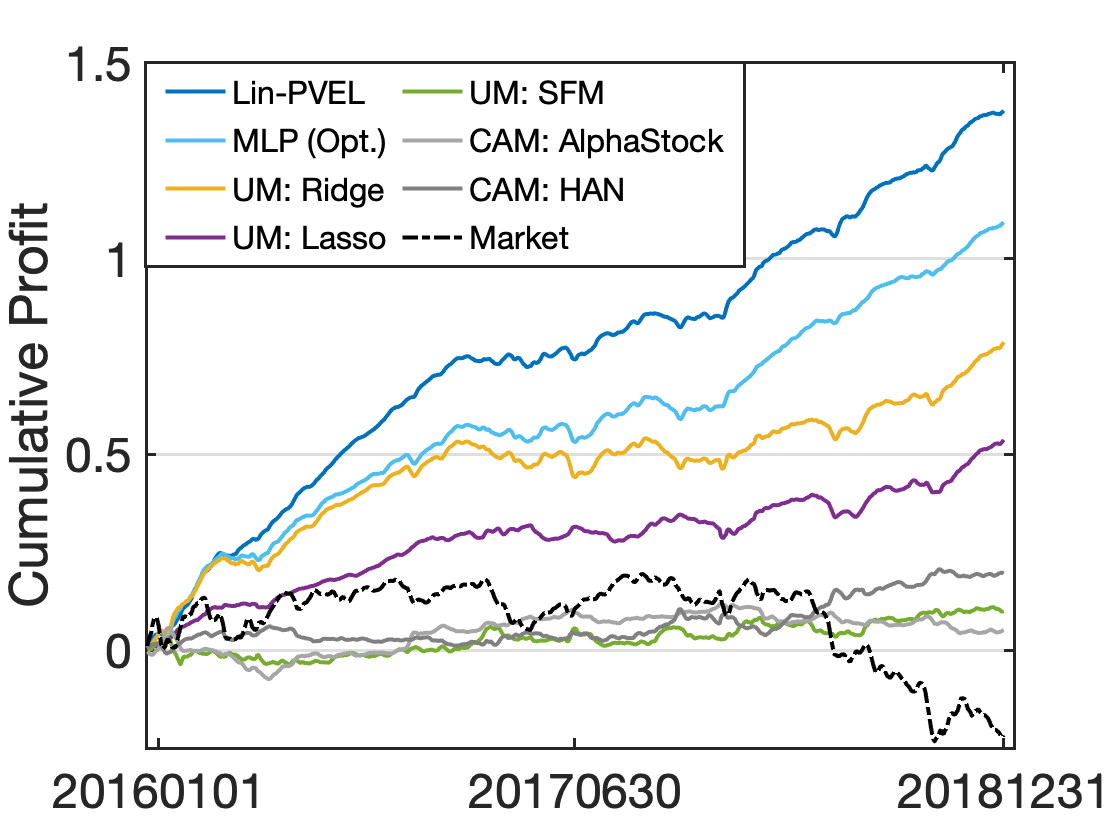}
    \label{appfig:quant_ls_3000} }     
    \subfigure[\textit{Weighted-long-short}.]{
    \includegraphics[width=0.3\textwidth]{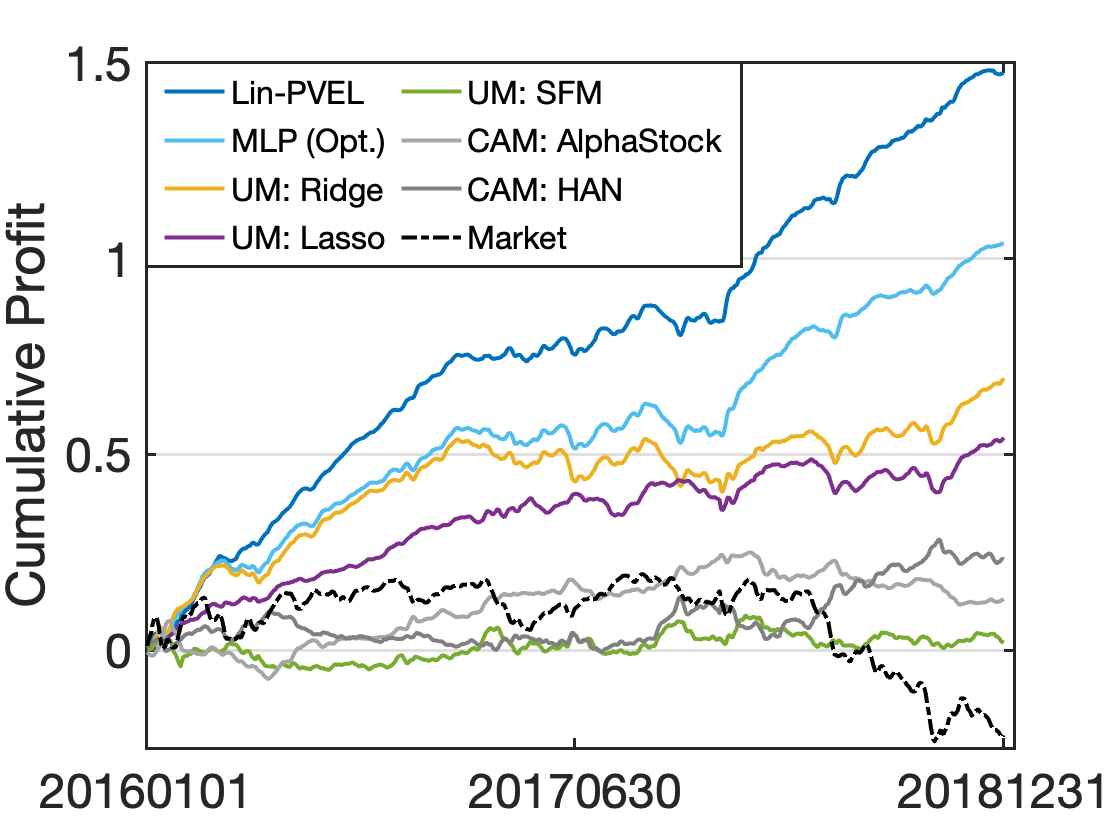}
    \label{appfig:quant_wls_3000}}
    \caption{Cumulative PnL (Profit \& Loss) curves of the top quintile portfolio  (i.e., on any given day, we consider a portfolios with only the top 20\% strongest in magnitude predictions, against future market excess returns). (a)-(c) are for the \emph{Universe 800} and (d)-(f) are for the \emph{Full universe}.}
    \label{appfig:fret_quantiles}
\end{figure*}

\myparab{Visualization/Qualitative examination}

\myparab{(1) Visualization for learned stock latent space.} We examine the latent positions we learned, and draw two observations.  \emph{(i)  Latent positions are not driven by sectors.} One possible explanation of our models' forecasting power is that they capture sector-related signals, e.g., growth of one airline implies the growth of others. Our visualization in Fig.~\ref{fig:my_label} shows this is not the case. \emph{(ii). Interactions are fine-grained.} We also present the neighbors uncovered by our pipeline, and also those found by \textsc{AlphaStock} for five stocks (all well known to the public). Our algorithm picks up different embeddings for these five stocks compared to \textsc{AlphaStock}, which indicates we discover an orthogonal signal.

\begin{figure*}[ht!]
\vspace{-1mm}
\centering
\includegraphics[scale=0.4]{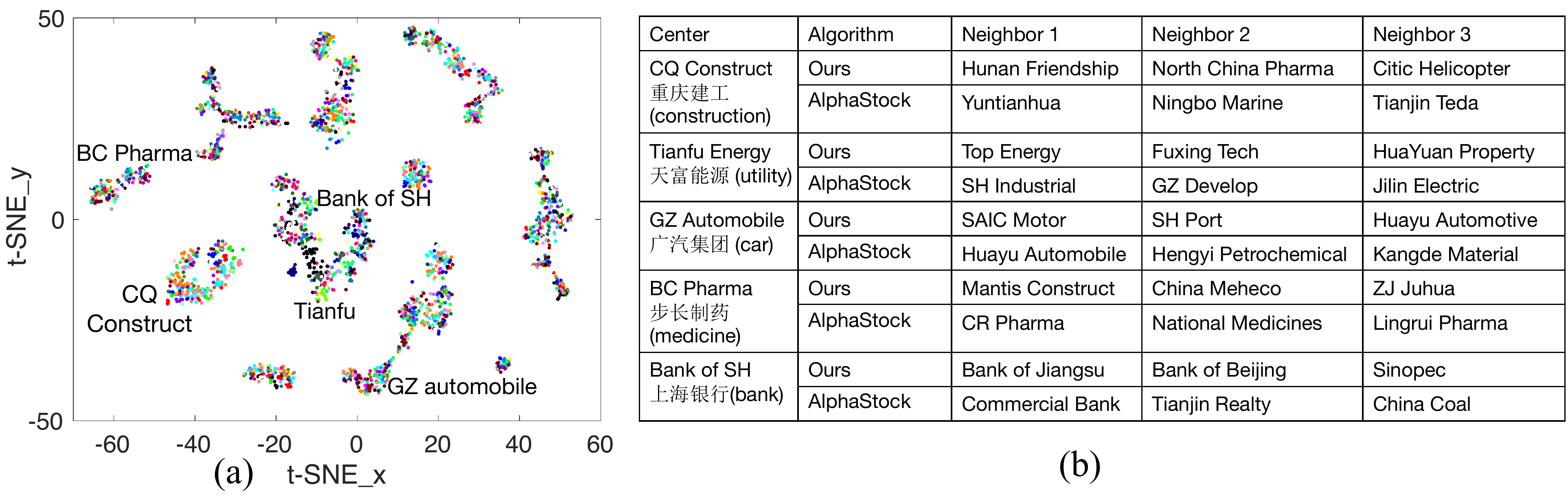}
\vspace{-4mm}
\captionsetup{width=1.02\linewidth}
\captionsetup{font=small}
\caption{{\small (a): t-SNE for our latent embedding (colors are coded by sectors); (b): Examples of stocks and their neighbors.}}
\label{fig:my_label}
\vspace{-4mm}
\end{figure*}

\subsection{Predicting user popularity in Twitter dataset}\label{Twitter_dataset}

We also evaluated our mode on the Twitter dataset and focus on predicting a user's next 5-day popularity. The popularity is defined as the sum of received quotes, retweets, and replies. 

\myparab{Data collection}
We used Twitter streaming API and tracked the tweets with topics related to the political keywords “trump”, “clinton”, “kaine”, “pence”, and “election2016”.
In total, we collected 15 months the Twitter data from October 01, 2016, to December 31, 2017, from 19 million distinct users and 804 million tweets.
The user $u$'s interaction is defined as is and only if he or she is quoted/replied/retweeted by another use $v$. Due to the huge size, we extract the subset of 2000 users with the most interactions for evaluation.

\myparab{Training and hyper-parameters}
We used October 01, 2016, to June 30, 2017, as the training period, July 01, 2017, to September 30, 2017 as the validation dataset to tune the hyper-parameters, and October 01, 2017, to December 31, 2017, as the testing dataset to evaluate the models' performance.

\begin{table}[h!]
\centering
\begin{adjustbox}{width=0.8\columnwidth,center}
\begin{tabular}{l|l|l|l|l} 
\hline
Models                   & MSE (in-sample) & MSE (out-of-sample) & Corr (in-sample) & Corr (out-of-sample)  \\ 
\hline
Ours: Lin-PVEL           & 0.472           & \textbf{0.520}      & 0.733            & \textbf{0.712}        \\
Ours: nparam-gEST        & 0.492           & \textbf{0.559}      & 0.688            & \textbf{0.658}        \\
Ours: MLP                & 0.486           & \textbf{0.547}      & 0.716            & \textbf{0.692}        \\
Ours: LSTM               & 0.484           & \textbf{0.541}      & 0.724            & \textbf{0.703}        \\ 
\hline
\hline
UM: Poor man Lin-PVEL    & 0.488           & 0.552               & 0.710            & 0.684                 \\
UM: Poor man nparam-gEST & 0.544           & 0.584               & 0.634            & 0.605                 \\
UM: Poor man MLP         & 0.506           & 0.562               & 0.703            & 0.673                 \\
UM: Poor man LSTM        & 0.496           & 0.559               & 0.710            & 0.679                 \\
UM: Linear models        & 0.616           & 0.663               & 0.618            & 0.592                 \\
UM: Random forest        & 0.611           & 0.659               & 0.623            & 0.587                 \\
UM: Xgboost              & 0.530           & 0.571               & 0.671            & 0.647                 \\ 
\hline
CEM: VR                  & 0.540           & 0.729               & 0.649            & 0.408                 \\
CEM: ARRR                & 0.564           & 0.652               & 0.610            & 0.573                 \\
Ad-hoc: Node2Vec~\cite{grover2016node2vec}         & 0.537           & 0.690               & 0.693            & 0.468                 \\ 
\hline
Consolidated: All Ours       & \textbf{0.459}  & \textbf{0.502}      & \textbf{0.767}   & \textbf{0.742}        \\
\hline
\end{tabular}
\end{adjustbox}
\vspace{-2mm}
\captionsetup{font=small}
\captionsetup{width=1.02\linewidth}
\caption{Overall in-sample and out-of-sample performance on the Twitter data set. Boldface denotes the best performance in each group. }
\label{table:twitter_app}
\vspace{-7mm}
\end{table}
\vspace{-1mm}


\newpage
\twocolumn
\bibliography{a_reference}

\end{document}